\documentclass{article} 

\usepackage[main, preprint]{neurips_2026}

\usepackage[utf8]{inputenc} 
\usepackage[T1]{fontenc}    
\usepackage{hyperref}       
\usepackage{url}            
\usepackage{booktabs}       
\usepackage{amsfonts}       
\usepackage{nicefrac}       
\usepackage{microtype}      
\usepackage{xcolor}         
\usepackage{wrapfig}
\usepackage{subfig}
\usepackage{siunitx}
\usepackage{enumitem}

\usepackage{amsmath}
\usepackage{amssymb}
\usepackage{mathtools}
\usepackage{amsthm}

\theoremstyle{plain}
\newtheorem{theorem}{Theorem}[section]
\newtheorem{proposition}[theorem]{Proposition}

\newtheorem{corollary}[theorem]{Corollary}
\theoremstyle{definition}
\newtheorem{definition}[theorem]{Definition}

\theoremstyle{remark}

\usepackage{algorithm}
\usepackage{algorithmic}

\title{Flow-Based Conformal Predictive Distributions}

\author{%
  Trevor A. Harris \\
  Department of Statistics \\
  University of Connecticut\\
  Storrs, CT 06269 \\
  \texttt{trevor.a.harris@uconn.edu} \\
}

\begin{document}
\maketitle

\begin{abstract}


Conformal prediction provides a distribution-free framework for uncertainty quantification via prediction sets with exact finite-sample coverage. In low dimensions these sets are easy to interpret, but in high-dimensional or structured output spaces they are difficult to represent and use, which can limit their ability to integrate with downstream tasks such as sampling and probabilistic forecasting. We show that any sufficiently regular differentiable nonconformity score induces a deterministic flow on the output space whose trajectories converge to the boundary of the corresponding conformal prediction set. This leads to a computationally efficient, training-free method for sampling conformal boundaries in arbitrary dimensions. Mixing across confidence levels yields conformal predictive distributions whose quantile regions coincide with the empirical conformal prediction sets. We provide an approximation bound decomposing CPD predictive error into score-induced distortion, base-measure quality, and gradient flow-induced distortion. We evaluate the approach on PDE inverse problems, precipitation downscaling, climate model debiasing, and hurricane trajectory forecasting.


\end{abstract}

\section{Introduction}

Conformal prediction is a general, distribution-free framework for post-hoc uncertainty quantification (UQ) \citep{vovk2005algorithmic, shafer2008tutorial, lei2018distribution}. Given a predictive model and a nonconformity score, conformal methods return prediction sets that are guaranteed to contain the true outcome with a prescribed $1 - \alpha$ probability. This guarantee holds even in finite samples without any distributional assumptions or restrictions on the structure of the predictive model. Because conformal methods are so widely applicable and provide rigorous statistical guarantees they have become enormously popular in modern machine learning settings \citep{angelopoulos2023conformal}, which increasingly rely on large complex models that lack native UQ mechanisms.

Both the strength and fundamental limitation of conformal methods is that they return prediction sets, rather than predictive distributions. Given a confidence level $\alpha \in (0, 1)$ and an input covariate $x \in \mathcal{X}$, a conformal algorithm will only return a label set $C_\alpha(x) \subset \mathcal{Y}$. This set is defined implicitly as a sublevel set of the chosen nonconformity score $S:\mathcal{X} \times \mathcal{Y} \to \mathbb{R}$, i.e. $C_\alpha(x) = \{y \in\mathcal{Y} : S(x, y)\le \tau_\alpha \}$ where $\tau_\alpha$ is an empirical threshold that guarantees $P(y \in C_\alpha(x)) \geq 1 - \alpha$. Pure set-wise UQ is advantageous because it is completely distribution free as long as the $(x, y)$ process is exchangeable. However, this also poses two significant issues. The first is how to even represent and apply these sets when the space $\mathcal{Y}$ is high-dimensional or a complex, structured manifold (Figure \ref{fig:motivation}). The second is how to compose prediction sets with downstream statistical tasks such as probabilistic forecasting, risk estimation, and simulation.

In a univariate space $\mathcal{Y} \subseteq \mathbb{R}$, conformal sets are representable as intervals, or collections of intervals, due to the natural ordering of $\mathbb{R}$. However, for any higher-dimensional space, which is not naturally ordered, there is no canonical way to represent $C_\alpha(x) = \{y \in\mathcal{Y} : S(x, y)\le \tau_\alpha \}$ unless the score has a very simple geometry (e.g. $\ell_2$ norms induce spherical level sets). Without a usable representation, the set $C_\alpha(x)$ cannot provide any future predictive uncertainty beyond membership queries. There are several attempts to provide a representation including prediction bands \citep{diquigiovanni2022conformal, ajroldi2023conformal}, generative modeling and transport maps \citep{wang2022probabilistic, ndiaye2025beyond}, and conformal ensembles \citep{harris2024quantifying, harris2025locally}. However, these approaches suffer from data limitation issues (ensembles) and heavy computation (OT maps, generative models), or greatly simplifying assumptions (bands). In domains such as geospatial modeling \citep{mao2024valid, harris2024quantifying}, image and video prediction \citep{ajroldi2023conformal, wen2025conformal, adame2025image}, operator and PDE modeling \citep{diquigiovanni2022conformal, ma2024calibrated, mollaali2024conformalized}, language modeling \citep{quach2023conformal, campos2024conformal, cherian2024large}, and more \citep{chernozhukov2021distributional, teneggi2023trust}, conformal prediction sets can become very challenging to represent, even with generative models, due to their complex structure and, often, limited calibration data.


\begin{figure}
\centering
\includegraphics[width=0.99\textwidth]{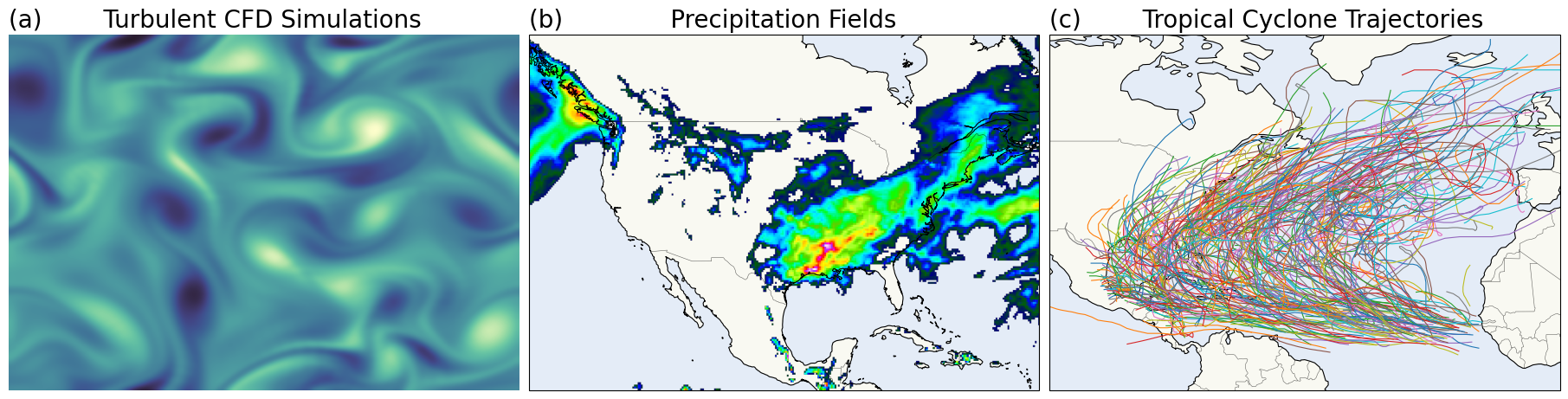}
\caption{What do prediction sets look like for CFD simulations? Precipitation patterns? Tropical Cyclones? How can we take implicitly defined sets and make them useful in these spaces?}
\label{fig:motivation}
\end{figure}

Even if a representation can be constructed, applying prediction sets to probabilistic tasks is still problematic. Common modeling tasks, such as probabilistic forecasting and risk estimation, intrinsically require interaction with a distribution to calculate probabilities and compute expectations. If the outcomes are univariate, $\mathcal{Y} \subseteq \mathbb{R}$, then conformal predictive systems \citep{vovk2017nonparametric, vovk2018cross, vovk2019universally} provide an analytically tractable approach for lifting prediction sets into well calibrated predictive distributions. However, these approaches do not immediately extend to higher dimensional spaces and manifolds without a natural ordering. Recent approaches extend these ideas to higher-dimensional spaces through generative modeling \citep{wang2022probabilistic, teneggi2023trust, zheng2024generative} or optimal transport (OT) \citep{ndiaye2025beyond}. These approaches are computationally heavy, introduce structure, and explicit OT maps scale poorly with dimension \citep{cuturi2013sinkhorn}.

In this work, we propose a constructive alternative based on flows \citep{robinson1998dynamical}. We show that any well-behaved differentiable nonconformity score can induce a deterministic dynamical system on $\mathcal{Y}$ whose trajectories globally attract to the conformal set boundary. The resulting flow, called a nonconformity flow, is uniquely determined up to tangential motion by the gradient of the nonconformity score, and, therefore, requires no additional training, auxiliary modeling, or transport mapping. We show that, for almost any initialization, the flow trajectories converge exponentially fast to the conformal boundaries in arbitrary dimensions under mild regularity conditions.
Therefore, nonconformity flows provide a training free, computationally efficient, and scalable mechanism for sampling conformal prediction set boundaries.

Furthermore, by mixing across confidence levels we can induce calibrated predictive distributions, called conformal predictive distributions (CPDs), for which every empirical conformal prediction set $C_\alpha(x)$ is an exact quantile region. CPDs are a scalable Monte Carlo alternative to analytical conformal predictive systems \citep{vovk2017nonparametric} and optimal transport maps \citep{ndiaye2025beyond} for lifting conformal prediction from sets to predictive distributions. 
%
We evaluate nonconformity flows and CPDs on a range of structured regression tasks, including elliptic PDE inverse problems, precipitation downscaling, climate model debiasing, and trajectory forecasting (Section \ref{sec:experiments}). We show that nonconformity flows quickly and reliably converge across a range of data modalities and score functions. 
The resulting CPDs are shown to be competitive with baseline uncertainty quantification (UQ) methods and can even act as controllable rare-event simulators (Figure \ref{fig:hurr}).


\section{Background} \label{sec:background}


Let $f_\theta: \mathcal{X} \to \mathcal{Y}$ denote a regression algorithm from the space $\mathcal{X}$ to an output space $\mathcal{Y}$, parameterized by $\theta \in \Theta$, and let $x \in \mathcal{X}$ and $y \in \mathcal{Y}$ denote specific covariates and targets, respectively. We allow $\mathcal{Y}$ to be infinite-dimensional (e.g., functions, fields, or operators), but assume that the method operates on a finite-dimensional differentiable representation of $\mathcal{Y}$ (e.g., via discretization or basis expansion), so that gradients of the nonconformity score and dynamical systems on the representation space are well-defined. We denote $\mathcal{D}_{train} = \{(x_s, y_s)\}_{s = 1}^m$ as the training dataset and $\mathcal{D}_{cal} = \{(x_t, y_t)\}_{t = 1}^n$ as a disjoint calibration dataset. All regression pairs $(x_t, y_t)$ are assumed to be exchangeable.

\subsection{Conformal Prediction} \label{sec:conformal_prediction}


Given a predictive model $f_\theta: \mathcal{X} \to \mathcal{Y}$ and a miscoverage level $\alpha \in (0, 1)$, conformal methods construct sets $C_\alpha(x) \subset \mathcal{Y}$ such that 
\begin{equation}\label{eqn:coverage_guarantee}
    \mathbb{P} \big(y  \in C_\alpha(x) \big)\ \ge\ 1-\alpha,
\end{equation}
where $(x, y)$ is new input that is exchangeable with the data in $\mathcal{D}_{cal}$. Remarkably, the validity criterion (Equation \ref{eqn:coverage_guarantee}) requires no assumptions about the model $f_\theta$ itself, and is obtained purely through an exchangeability assumption \citep{shafer2008tutorial}. Given the model and the calibration data, constructing $C_\alpha(x)$ is also straightforward and computationally efficient.

In the standard inductive, or split, conformal prediction we choose a positively oriented scoring function $S:\mathcal{X} \times \mathcal{Y} \to \mathbb{R}$, such as $S(x, y) = || y - f_\theta(x)||_2$, which measures the nonconformity of $f_\theta(x)$ with the target $y$. We then compute the nonconformity scores of all calibration data to obtain the sequence $S_1 = S(x_1, y_1), \dots, S_n = S(x_n, y_n)$. Let $S_{(1)}\le \cdots \le S_{(n)}$ be the order statistics of the calibration scores. The conformal threshold is defined as $\tau_\alpha = S_{(k)}$ where $k=\lceil (1-\alpha)(n+1)\rceil$. Given the $\tau_\alpha$, we can define the conformal prediction set 
\begin{equation}\label{eq:split-conformal-set}
    C_\alpha(x) \;=\; \big\{\, y \in\mathcal{Y}:\ S(x, y)\le \tau_\alpha\,\big\},
\end{equation}
which is guaranteed to meet the validity criterion (Equation \ref{eqn:coverage_guarantee}) when the data are fully exchangeable. For our method, we do not restrict the nonconformity score $S(x, y)$, except that its gradient with respect to $y$, $\nabla S(x, y)$, exists almost everywhere for all $y \in \mathcal{Y}$, and that $C_\alpha(x)$ is non-empty, i.e. $\inf_y S(x, y) \leq \tau_\alpha$. This holds for all standard residual-based scores since $S(x, f_\theta(x)) = 0 < \tau_\alpha$. We generally suppress $x$ and write $S(y)$.

\subsection{Dynamical Systems and Flows} \label{sec:dynamical_systems}



A dynamical system is a rule for deterministically evolving states forward in time \citep{robinson1998dynamical}. A continuous time dynamical system is defined as a pair $(\mathcal{Y}, \Phi)$ where $\Phi : \mathbb{R} \times \mathcal{Y} \to  \mathcal{Y}$ is a continuous flow. Continuous flows are functions satisfying $\Phi(0, y) = y ;\ \forall y \in \mathcal{Y}$, $\Phi(t + s, y) = \Phi(t, \Phi(s, y))$, and $\Phi(\cdot)$ continuous.
We can induce a dynamical system $(\mathcal{Y}, \Phi)$ on the state space $\mathcal{Y}$ through an ordinary differential equation (ODE) defined by a vector field $v : \mathcal{Y} \to \mathcal{Y}$ as:
\begin{equation} \label{eqn:flow_ode}
	 y'(t) = v(y(t)),
\end{equation}
where $y(t)$ denotes the system state at time $t \in \mathbb{R}$ and $y'(t)$ its time derivative. The induced flow, $\Phi(t, y_0)$, is the map computing the solution of the ODE (Equation \ref{eqn:flow_ode}) at time $t$, i.e. 
\begin{equation} \label{eqn:flow}
	\Phi(t, y_0) = y(t; y_0) = y_0 + \int_0^t v(y(s))ds,
\end{equation}
where $y_0 = y(0)$ is the initial condition. Flows are widely used for scientific modeling  \citep{robinson2012introduction, hirsch2013differential} and, more recently, generative modeling \citep{dinh2016density, chen2018neural, liu2022flow, lipman2023flow}.

An important concept in dynamical systems is attraction. An attractor is a region of $\mathcal{Y}$ that $\Phi(t, \cdot)$ approaches as $t \rightarrow \infty$ and then never leaves. More formally, an attractor of $\Phi(t, \cdot)$ is defined as a subset $A \subset \mathcal{Y}$ such that $\Phi(t, A) = A$ for all $t \geq 0$ (invariance) and $\lim_{t \rightarrow \infty} \text{dist}(\Phi(t, U),  A) = 0$ (attraction) where the open subset $U \subset \mathcal{Y}$ is the basin of attraction. 

\section{Method} \label{sec:method}

As in Section \ref{sec:background}, let $f_\theta : \mathcal{X} \to \mathcal{Y}$ denote our regression algorithm, $\mathcal{D}_{train} = \{(x_s, y_s)\}_{s = 1}^m$ the training data, $\mathcal{D}_{cal} = \{(x_t, y_t)\}_{t = 1}^n$ the calibration data, and $S: \mathcal{X} \times \mathcal{Y} \to \mathbb{R}$ a nonconformity score whose gradient at $y \in \mathcal{Y}$, denoted $\nabla S(x, y)$, exists almost everywhere. We assume that all pairs $(x_t, y_t)$ are exchangeable. 
For a target miscoverage level $\alpha \in (0,1)$, let $\tau_\alpha$ denote the split conformal threshold and define the conformal prediction set as 
$C_\alpha(x) = \{y \in \mathcal{Y} : S(x, y) \leq \tau_\alpha \}$ and its boundary 
as $\partial C_\alpha(x) = \{y \in \mathcal{Y} : S(x, y) = \tau_\alpha \}$.

\subsection{Sampling conformal boundaries} \label{sec:levelset_sampling}

First, we introduce nonconformity flows to efficiently sample from the conformal boundaries $\partial C_\alpha(x) = \{y \in \mathcal{Y} : S(x, y) = \tau_\alpha \}$ in arbitrary dimensions given a differentiable score.

\begin{figure} 
\centering
\includegraphics[width=0.99\textwidth]{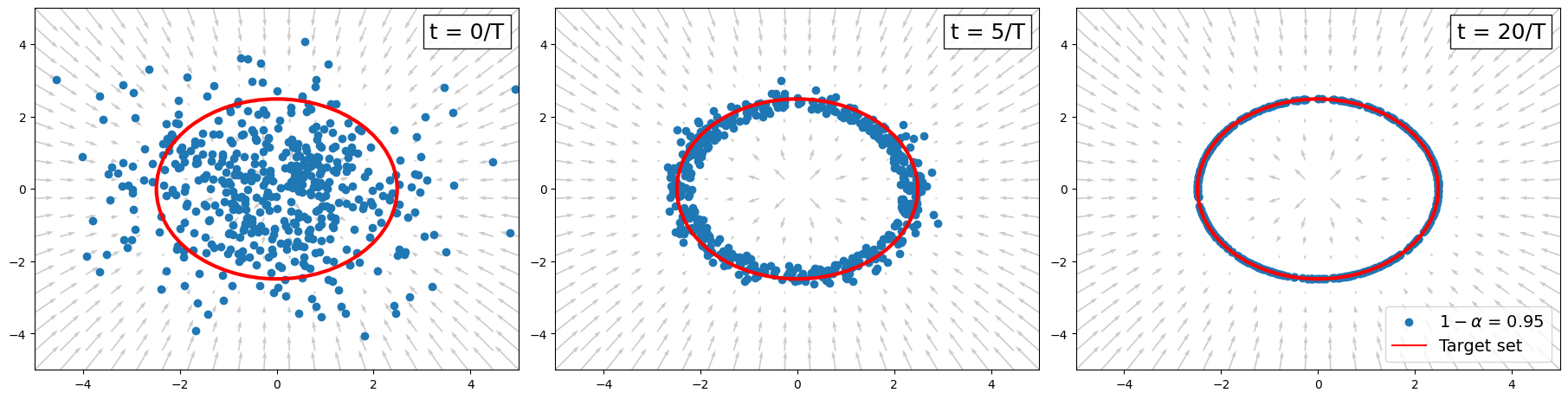}
\caption{The nonconformity flow (arrows) globally attracts towards the target level set ($\ell_2$ score).}
\label{fig:flowmap}
\end{figure}

\paragraph{Nonconformity flows} We define a continuous time dynamical system, $(\mathcal{Y}, \Phi)$, on $\mathcal{Y}$ through the following ODE system:
\begin{align} \label{eqn:score_flow}
   y'(t)&= v(y(t)) \\
   S'(y(t)) &= -\lambda(S(y(t)) - \tau_\alpha)
\end{align}
where $y(t) \in \mathcal{Y}$ is a trajectory in $\mathcal{Y}$ and $S(y(t)) = S(f_\theta(x), y(t))$ is the score of state $y(t)$ with $f_\theta(x) = \hat y \in \mathcal{Y}$ held fixed. As in Section \ref{sec:dynamical_systems}, $v: \mathcal{Y} \times \mathbb{R} \to \mathcal{Y}$ is a vector field that governs the induced flow.
The second equation, however, is a linear ODE on the score that controls the behavior of $v(\cdot)$, which we denote as the score controller. Among all vector fields satisfying the score controller, we select the minimum norm solution at each time $t$. Whenever $\nabla S(y(t)) \neq 0$, this yields an explicit vector field
\begin{equation} \label{eqn:frontier_velocity}
    v_\alpha(y(t)) = -\lambda(S(y(t))- \tau_\alpha) \frac{\nabla S(y(t))}{\Vert \nabla S(y(t)) \Vert^2_2},
\end{equation}
where $\Vert \cdot \Vert_2^2$ is the squared Euclidean norm on $\mathcal{Y}$. This field is derived in Appendix \ref{apx:sec:derivation}. Importantly, $v_\alpha(y(t))$ is computable directly from the score and does not require training, modeling, or an OT map. The minimum norm solution eliminates tangential motion on the induced flow
\begin{equation} \label{eqn:flow_sampling}
	\Phi_\alpha(t, y_0)  = y_0 + \int_0^t v_\alpha(y(s))ds,
\end{equation}
meaning trajectories efficiently follow the gradient field lines directly to the target level set. We say this flow is score controlling since points $y(t)$ with scores $S(y(t)) > \tau_\alpha$ are pushed inwards, while points such that $S(y(t)) < \tau_\alpha$ are pushed outwards to the target level set $\partial C_\alpha(x)$ (Figure \ref{fig:flowmap}). Once $S(y(t)) = \tau_\alpha$ then $v_\alpha(\cdot)$ vanishes and motion stops. The free parameter $\lambda > 0$ controls the flow velocity and its value can be set automatically to achieve $\epsilon$-convergence in finite time (Corollary \ref{prop:hitting_time}).

\paragraph{Convergence} We show that the induced flow $\Phi_\alpha(t, y_0)$ converges globally and exponentially fast to the boundary set $\partial C_\alpha(x)$. Moreover we upper bound the pointwise convergence rate of a trajectory $y(t) \in \mathcal{Y}$ to its unique limit point in $\partial C_\alpha(x)$ and show that convergence is dimension free, aside from the higher cost of evaluating $\nabla S(y(t))$. Proposition \ref{prop:convergence} makes this exact.

\begin{proposition}[Convergence]\label{prop:convergence}
Define the score error $\varepsilon(t):=S(y(t))-\tau_\alpha$. Assume $y(t)$ exists for all $t \geq 0$ and $\nabla S(y(t)) \neq 0$ whenever $S(y(t)) \neq \tau_\alpha$. Then:
\begin{enumerate}
	\item Score convergence:
	$\varepsilon'(t)=-\lambda\,\varepsilon(t)$, hence
	$\varepsilon(t)=\varepsilon(0)\,e^{-\lambda t}$
	and
	$S(y(t))\to\tau_\alpha \ \text{as } t\to\infty$.
	In particular, every $\omega$-limit point of $(y(t))_{t\ge 0}$ lies in $\partial C_\alpha(x)$.

	\item Pointwise convergence:
	Assume, in addition, that there exist $T<\infty$ and $m>0$ such that $ \|\nabla S(y(t))\|_2 \ge m$ for all $t\ge T$, then $y(t)$ converges to a unique limit point $y_\infty \in \partial C_\alpha(x)$ at the exponential rate
	\[
		\|y(t)-y_\infty\|_2
		\le
		\frac{1}{m}\,|S(y_0)-\tau_\alpha|\,e^{-\lambda t}
		\qquad\text{for all } t\ge T.
	\]
\end{enumerate}
\end{proposition}
Proof in Appendix \ref{apx:sec:proofs}. Proposition \ref{prop:convergence} states that, under mild regularity, for almost any base sample $y_0 \in \mathcal{Y}$, the induced flow will attract to the target level set so that terminal samples will always be elements of $\partial C_\alpha(x)$. This property is completely induced by the controller equation. In fact, because the score controller is a linear ODE, by part 1 in Proposition \ref{prop:convergence}, we get the following corollary describing the finite $\varepsilon$-hitting time of the flow.
\begin{corollary}[$\varepsilon$-hitting time] \label{prop:hitting_time}
For $\varepsilon>0$ and a base sample $y_0$, define the (first) $\varepsilon$-hitting time
$T_\varepsilon(y_0) = \inf \{t \ge 0:\ |S(\Phi_\alpha(t, y_0))-\tau_\alpha|\le \varepsilon \}$.
The $\varepsilon$-hitting time of $y_0$ under $v_\alpha(\cdot)$ is
\begin{equation}
T_\varepsilon(y_0) = (1/\lambda) \log \big[ \vert S(y(0)) - \tau_\alpha \vert / \varepsilon \big]
\end{equation}
\end{corollary}
The $\varepsilon$-hitting time tells us how long it will take for the sample to be within $\varepsilon \geq 0$ of the target level set. In practice, we can use this to set the velocity scalar $\lambda$ to achieve near convergence in a fixed amount of time, e.g. $t = 1$ by taking $\lambda = \max\!\big(\log (|S(y(0)) - \tau_\alpha| / \varepsilon),\, 0\big)$ with a predefined tolerance $\varepsilon > 0$; points already within tolerance are left unchanged. We use this rule to set $\lambda$ in all numerical experiments (Section \ref{sec:experiments}) using $t = 1$ and $\varepsilon=10^{-6}$.

\subsection{Conformal Predictive Distributions}\label{sec:cpds}

The defining feature of conformal methods is that they produce prediction sets.
This set-valued output enables distribution-free validity without assuming a likelihood or parametric model (Section \ref{sec:background}), but it does not directly provide a probabilistic description of uncertainty. In many downstream tasks, such as probabilistic forecasting, forward simulation, and risk assessment, sets indexed by confidence levels are insufficient, and a single predictive distribution is typically required.

This motivates the problem of lifting conformal prediction from sets to distributions. Conformal methods can provide partial information about a predictive distribution through sequences of prediction sets $\{C_\alpha(X)\}_{\alpha \in (0, 1)}$ where $C_{\alpha_1} \subseteq C_{\alpha_2}$ for $\alpha_1 \geq \alpha_2$,
though they impose no additional constraints beyond this filtration. As a result, conformal prediction cannot identify a unique predictive measure, but rather a family of distributions all compatible with the filtration. We formalize this notion of compatible conformal distributions and provide a constructive algorithm to sample from them.

Fix $x \in \mathcal{X}$ and let $\{C_\alpha(X)\}_{\alpha \in (0, 1)}$ denote the conformal filtration induced by score $S : \mathcal{X} \times \mathcal{Y} \to \mathbb{R}$. 
\begin{definition}[Admissibility]
Let $\pi$ be a probability measure on $(0,1)$, and let
$\mathcal A_{\mathrm{conf}}=\{\alpha_k=1-k/(n+1):k=1,\ldots,n\}$
denote the empirical conformal grid. We say that a predictive distribution
$P_x$ is $\pi$-admissible with respect to $S(\cdot)$ if, for all
$\alpha\in(0,1)$,
\[
P_x(C_\alpha(x)) \geq \pi([\alpha,1)),
\]
and, for all $\alpha_k\in\mathcal A_{\mathrm{conf}}$,
\[
P_x(C_{\alpha_k}(x)) = \pi([\alpha_k,1)).
\]
\end{definition}
$\pi$-Admissibility formalizes what it means for a distribution to be consistent with the conformal filtration. A natural choice for $\pi$ is $\text{Unif}(0, 1)$, in which case $P_x(C_\alpha(x)) = 1-\alpha$ on every empirical conformal set, i.e. each observed $C_\alpha(x)$ is an exact quantile region of $P_x$. Non-uniform $\pi$ measures can still be useful for targeted sampling, e.g. rare event sampling (Section \ref{sec:distribution_quality}).

We can construct $\pi$-admissible predictive measures $P_x$ by mixing over confidence levels. For each $\alpha \in (0, 1)$, 
let $\nu_{x, \alpha}$ be a probability measure supported on $\partial C_\alpha(x)$. We assume that for all measurable $A \subseteq \mathcal{Y}$, the map $\alpha \mapsto \nu_{x, \alpha}$ is also measurable. Given a mixing measure $\pi$ on $(0, 1)$, we define a conformal predictive distribution (CPD) by
\begin{equation} \label{eqn:conformal_distribution}
    P_x^{\mathrm{CPD}}(A) = \int_0^1 \nu_{x, \alpha}(A)d\pi(\alpha), \qquad A \in \mathcal{B}(\mathcal{Y}),
\end{equation}
This construction does not require any likelihood model, transport map, or parametric assumptions. It is based purely on the conformal score and the choice of mixing measure $\pi$.
Furthermore, by design, all CPDs are admissible regardless of the chosen boundary measure.
\begin{proposition}[Calibration]\label{prop:cpd-calibration}
Let $P_x^{\mathrm{CPD}}$ denote a CPD constructed from boundary measures
$\{\nu_{x,a}\}_{a\in(0,1)}$ and mixing measure $\pi$. Then, for every
$\alpha\in(0,1)$,
$
P_x^{\mathrm{CPD}}(C_\alpha(x))
=
\pi\{a:\tau_a\le \tau_\alpha\}
\ge
\pi([\alpha,1)).
$
Moreover, for every empirical conformal level
$\alpha_k\in\mathcal A_{\mathrm{conf}}$,
\[
P_x^{\mathrm{CPD}}(C_{\alpha_k}(x))
=
\pi([\alpha_k,1)).
\]
In particular, if $\pi=\mathrm{Unif}(0,1)$, then $P_x^{\mathrm{CPD}}(C_\alpha(x))\ge 1-\alpha$ for all $\alpha$, with equality on the empirical conformal grid $\mathcal A_{\mathrm{conf}}$.
\end{proposition}
Proof in Appendix \ref{apx:sec:proofs}. By Proposition \ref{prop:cpd-calibration}, any $P_x^{\mathrm{CPD}}$ is $\pi$-admissible with respect to the conformal filtration. When $\pi=\mathrm{Unif}(0,1)$, this gives $P_x^{\mathrm{CPD}}(C_\alpha(x)) = 1-\alpha$ for the empirical conformal prediction sets. This holds independently of the particular choice of boundary measures $\{\nu_{x, \alpha}\}_{\alpha \in (0, 1)}$, thus freeing us to choose them according to other criteria (e.g. optimize predictive skill).

\begin{wrapfigure}[9]{R}{0.5\textwidth}
 \begin{minipage}{0.5\textwidth}
 \vspace{-1.9em}
  \begin{algorithm}[H]
   \caption{CPD Sampling}\label{alg:sampling}
   \begin{algorithmic}
   \STATE \textbf{Input:} Score $S(\cdot)$, prediction $\hat y$.
   \STATE 1. Sample $\alpha \sim \pi$
   \STATE 2. Sample $y_0 \sim \mu_x$
   \STATE 3. Construct flow $ \Phi_\alpha$ given $S(\cdot)$ and $\hat y$.
   \STATE \textbf{Return:} $y = \lim_{t \rightarrow \infty} \Phi_\alpha(t, y_0)$
   \end{algorithmic}
  \end{algorithm}
 \end{minipage}
\end{wrapfigure}
In practice, we can specify $\nu_{x, \alpha}$ through a pushforward construction under our induced flow. Let $\mu_x$ be a base measure on $\mathcal{Y}$ and define 
\begin{equation} \label{eqn:hitting}
	\Pi_{x, \alpha}(y_0) = \lim_{t \rightarrow \infty} \Phi_\alpha(t, y_0)
\end{equation} 
as the measurable hitting map that transports initial points $y_0 \in \mathcal{Y}$ to the boundary set $\partial C_\alpha(x)$.
We take $\nu_{x, \alpha} = (\Pi_{x, \alpha})_ {\#}\mu_x$. Sampling under $P_x^{\mathrm{CPD}}$ then reduces to the following procedure: sample $\alpha \sim \pi$, sample $y_0 \sim \mu_x$, and output $y = \Pi_{x, \alpha}(y_0)$ (Alg. \ref{alg:sampling}). 
This procedure generalizes standard inverse CDF sampling: the confidence level $\alpha$ selects a contour of the conformal filtration, while $\nu_{x, \alpha}$ selects a point along that contour. The resulting distribution is a valid, $\pi$-calibrated representative of the $\pi$-admissible set $\mathcal{P}_{conf} = \{ P : P(C_\alpha) \geq \pi\big([\alpha, 1)\big) \;\forall \alpha \in (0, 1)\}$, and is conformally calibrated if $\pi=\mathrm{Unif}(0,1)$.

\subsection{Forecast skill of Conformal Predictive Distributions}
\label{sec:boundary_optimality}

CPDs are non-unique predictive distributions that are guaranteed to be compatible with the conformal filtration. However, different boundary measures, or, equivalently, different base measures, can produce CPDs with substantially different forecasting skill. 
The following proposition bounds the $W_2$ approximation error between the conditional distribution $P(\cdot \mid x)$ and the CPD distribution $P_x^{\text{CPD}}(\cdot)$. 
\begin{proposition}[$W_2$ approximation bound]\label{prop:w2_cpd}
Fix $x\in\mathcal X$ and let $P_x^{\mathrm{proj}} = \int_0^1(\mathrm{proj}_{x,\alpha})_\# P_x^\star \,d\alpha $ denote the distribution of $\mathrm{proj}_{x,\alpha}(Y)$ for $Y\sim P_x^\star = P(\cdot \mid x)$ and $\alpha\sim\mathrm{Unif}(0,1)$. Under the regularity conditions
stated in Appendix~\ref{proof:w2_cpd}, 
\[
W_2(P_x^\star, P_x^{\mathrm{CPD}})
\le
W_2(P_x^\star, P_x^{\mathrm{proj}})
+
(\mathbb E_\alpha L_\alpha^2)^{1/2}W_2(P_x^\star, \mu_x)
+
C\{\mathbb E_\alpha[\bar\kappa(\alpha)^2R_4(\alpha)]\}^{1/2},
\]
where $R_4(\alpha):=\mathbb E_{\mu_x}[\mathrm{dist}(y,\partial C_\alpha(x))^4]$,
$\bar\kappa(\alpha)$ controls the curvature of the score-gradient paths, and
$C>0$ is a geometric constant.
\end{proposition}

Proof in Appendix~\ref{apx:sec:proofs}. 
The first term, $W_2(P_x^\star, P_x^{\mathrm{proj}})$, is the cost of transporting the true conditional $P_x^\star = P(\cdot \mid x)$ into an admissible CPD. The smaller this distance is, the more compatible the CPD class is with the data in the sense that it can support distributions like $P(\cdot \mid x)$. We can keep this cost small by choosing scores whose level sets are close to the high density regions of $P(\cdot \mid x)$, i.e. $S(y) \approx g(-\log p_x(y))$, where $p_x$ is the density of $P(\cdot \mid x)$ and $g$ monotonic. Because $p_x$ is unknown in practice, this term represents the price we pay for imposing our chosen score on the data.
%

The second term, $(\mathbb E_\alpha [L_\alpha^2])^{1/2}W_2(P_x^\star, \mu_x)$, measures how close the base measure $\mu_x$ is to $P(\cdot \mid x)$ and the smoothness of projection onto the level sets. If $\mu_x \approx P(\cdot \mid x)$, then this term can be kept small. However, if the level sets are highly non-convex or sharply curved, then $E_\alpha [L_\alpha^2]$ can be large and amplify any error due to not sampling from the true conditional. 

The last term, $C (\mathbb E_\alpha[\bar\kappa(\alpha)^2R_4(\alpha)])^{1/2}$, measures the amount of distortion introduced by gradient flow sampling over direct projection sampling with $\mathrm{proj}_{x,\alpha}$. When the curvature of the velocity field, controlled by $\bar\kappa(\alpha)$, and distance to the level sets $R_4(\alpha)$ are both small, then this term will be small. Otherwise, if the flow paths are highly curved then the prediction error can be high, even if the score and base measure are well aligned with the data. 

Proposition \ref{prop:w2_cpd} essentially states that, in order to control the prediction error $W_2(P_x^\star, P_x^{\mathrm{CPD}})$, we should prefer smooth scores with convex, or nearly convex, level sets that are well-aligned with the conditional distribution. We should also initialize our flows as close to the conditional distribution as possible. In practice this may be difficult to achieve exactly, but norm based scores and initializing with calibration responses works well empirically (Section \ref{sec:experiments}).

We note that Proposition~\ref{prop:w2_cpd} immediately bounds the energy distance \citep{szekely2013energy}, a proper scoring rule \citep{gneiting2007strictly}, $D_E(P_x^\star, P_x^{\text{CPD}}) = 2\,\mathbb{E}_{Y \sim P_x^\star, Z \sim P_x^{\text{CPD}}}[\|Y {-} Z\|] -\mathbb{E}_{Z, Z'}[\|Z {-} Z'\|] - \mathbb{E}_{Y, Y'}[\|Y {-} Y'\|]$ via the inequality $D_E \leq 2\,W_2$. Thus, bounding Wasserstein distances helps control forecast error.

\section{Numerical Experiments}
\label{sec:experiments}
%

We evaluate the proposed framework along two axes (i) convergence and diversity of the nonconformity flow samples (Section \ref{sec:convergence}) and (ii) conformal predictive distributions (CPDs) as calibrated predictive objects (Section \ref{sec:distribution_quality}). Definitions of score functions, data generation, preprocessing, model architectures, and evaluation metrics are deferred to Appendices \ref{apx:sec:data} - \ref{apx:sec:metrics}. Additional simulations comparing computational cost and scaling are provided in Appendix \ref{apx:sec:scaling}.

\subsection{Evaluating Conformal Level Sets}
\label{sec:convergence}


\paragraph{Convergence} We first evaluate whether the proposed flow (Equation \ref{eqn:score_flow}) converges to within tolerance ($\varepsilon=10^{-6}$) with $\lambda$ set per trajectory according to the $\varepsilon$-hitting time equation (Corollary \ref{prop:hitting_time}). We consider isotropic Gaussian, anisotropic Student-t ($df = 3$), 2D Gaussian processes, and 2D GP upsampling targets across a range of settings and score functions
(Appendix \ref{apx:sec:data}). For the vector regression tasks, we train a multilayer perceptron (Appendix \ref{apx:sec:models}) and for the operator tasks, a 2D Fourier Neural Operator (Appendix \ref{apx:sec:models}). All models trained on a single A100 GPU for 25-100 epochs until convergence using the Adam optimizer with learning rate $\gamma=10^{-3}$.



\begin{figure}[h]
\centering
\includegraphics[width=\textwidth]{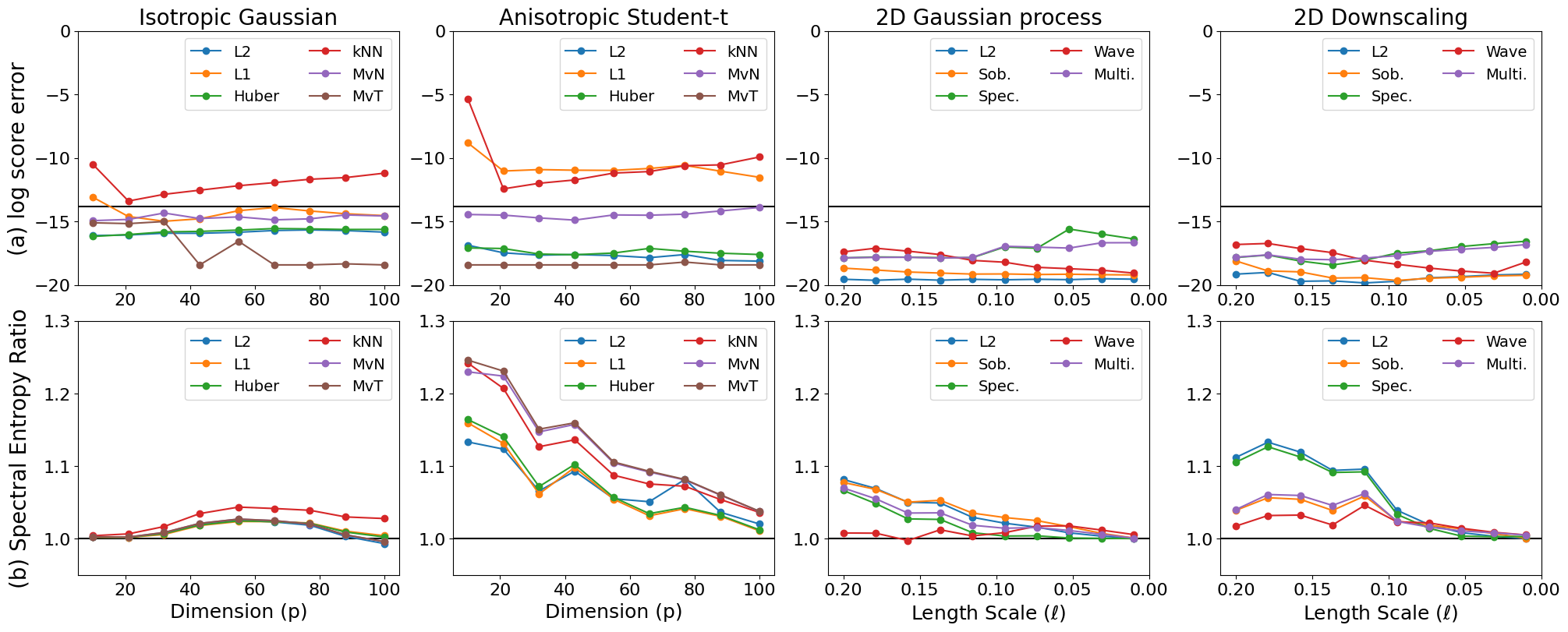}
\caption{\textbf{Top:} Log absolute convergence error averaged across all $\alpha$ levels. \textbf{Bottom:} Spectral entropy of the generated samples, normalized by the spectral entropy of the test data.}
\label{fig:convergence}
\end{figure}
Figure \ref{fig:convergence}a shows the log score error after 20 forward Euler integration steps ($dt = 1/20$) broken out by score function (Appendix \ref{apx:sec:vec_scores} -- \ref{apx:sec:2d_scores}) with the black line indicating the target tolerance. Across all settings and tasks the conformal flow sampler converges to within tolerance except the $\ell_1$ and $k$NN score, due to their non-smoothness. Figure \ref{fig:convergence}b shows the spectral entropy ratio \citep{friedman2022vendi}, which measures sample diversity relative to the observed data. All settings show that nonconformity flows generate samples that are at least as diverse as the target data (ratio $\geq 1$), thus the sampler is not collapsing to a low dimensional subspace even in high dimension. Additional simulations (Figure \ref{fig:diversity} in Appendix \ref{apx:sec:data}) show that even within a specific $\alpha$ level the flow produces diverse samples.

\subsection{Evaluating Conformal Predictive Distributions}
\label{sec:distribution_quality}

\paragraph{Distribution Approximation} We next compare CPDs to widely used machine learning ensemble baselines: MC Dropout \citep{gal2016dropout}, Deep Ensembles \citep{lakshminarayanan2017simple}, Mean-Variance Networks \citep{nix1994estimating}, Conditional Flow Matching \citep{lipman2023flow},  Mixture Density Networks \citep{bishop1994mixture}, and Implicit Quantile Networks \citep{dabney2018implicit}. Details of each method provided in Appendix \ref{apx:sec:models}. We consider five structured regression tasks: predict 2D Gaussian processes, inverse model 2D Elliptic PDE solutions, approximate Navier-Stokes solutions ($t = 2$), downscale North American precipitation fields, and debiasing a global climate model's temperature fields (Appendix \ref{apx:sec:data}). 
These latter two experiments use temporally correlated data, so exchangeability is violated and coverage is not guaranteed.
We consider both a global $\ell_2$ score (CPD-G) and local composite score (CPD-L) (Appendix \ref{apx:sec:2d_scores}).
We measure the quality of each predictive distribution over the test split via Energy Distance (ED) \citep{gneiting2007strictly}, Log Spectral Distance (LSD) \citep{jain1989fundamentals, portilla2000parametric}, and local MMD \citep{binkowski2018demystifying, amir2021understanding} (Appendix \ref{apx:sec:metrics}). ED measures distribution approximation,
LSD measures how energy is distributed over frequency scales, measuring spatial pattern similarity, while the local MMD measures local detail fidelity. All metrics are computed from $M = 20$ predictive samples per test instance.

\begin{table}[h]
    \caption{Comparison of predictive distribution and conformal predictive distribution (CPD) baselines. Bootstrap standard errors (Appendix Tables \ref{tab:baselines_group1}, \ref{tab:baselines_group2}) are generally below $10^{-2}$.}
    \resizebox{\linewidth}{!}{%
    \centering
    \small
    \setlength{\tabcolsep}{4pt}
    \begin{tabular}{l ccc ccc ccc ccc ccc ccc}
    \toprule
    	 & \multicolumn{3}{c}{\bf GP Regression} &
	     \multicolumn{3}{c}{\bf Elliptic PDE Inv.}  & 
	     \multicolumn{3}{c}{\bf Navier Stokes}  & 
	     \multicolumn{3}{c}{\bf Precip. Downscale} & 
	     \multicolumn{3}{c}{\bf Climate Debias} \\
	     \cmidrule(lr){2-4} 
	     \cmidrule(lr){5-7}
	     \cmidrule(lr){8-10}
	     \cmidrule(lr){11-13}
	     \cmidrule(lr){14-16}
    Method		& ED  & LSD & MMD  & ED  & LSD & MMD  & ED  & LSD & MMD  & ED  & LSD & MMD & ED  & LSD & MMD  \\
    \midrule
    CPD-G 		& 0.313  & 0.069 & 0.023 & 0.068 & 0.103  & 0.006 & 0.233 & 0.252 & 0.043 & 0.395 & 0.067 & 0.019 & 0.106 & 0.003 & 0.024  \\
    CPD-L 		& 0.313  & 0.069 & 0.036 & 0.067 & 0.099 & 0.006 & 0.229 & 0.166 & 0.039 & 0.400 & 0.066 & 0.016 & 0.130 & 0.006 & 0.035 \\
    Drop.         	& 0.400  & 0.083 & 0.187 & 0.115 & 0.263  & 0.158 & 0.295 & 0.245 & 0.142 & 0.418 & 0.072 & 0.075 & 0.127 & 0.004 & 0.210 \\
    D. Ens.		& 0.346  & 0.089 & 0.064 & 0.118 & 0.234  & 0.044 & 0.269 & 0.193 & 0.055 & 0.399 & 0.059 & 0.138 & 0.119 & 0.002 & 0.108 \\
    IQN             	& 0.427  & 0.075 & 0.061 & 0.139 & 0.262  & 0.072 & 0.263 & 0.381 & 0.181 & 0.393 & 0.066 & 0.148 & 0.113 & 0.002 & 0.123 \\
    Flow              & 0.410  & 0.138 & 0.167 & 0.113 & 0.351  & 0.040 & 0.235 & 0.182 & 0.048 & 0.374 & 0.065 & 0.093 & 0.118 & 0.003 & 0.113 \\
    MDN         	& 0.367  & 0.213 & 0.211 & 0.141 & 0.457 & 0.169 & 0.236 & 0.677 & 0.171 & 0.361 & 0.090 & 0.114 & 0.104 & 0.002 & 0.199 \\
    MVE		& 0.374  & 0.200 & 0.211 & 0.203 & 6.112  & 0.198 & 0.236 & 0.651 & 0.171 & 0.722 & 0.097 & 0.117 & 0.226 & 0.019 &  0.266 \\
    \bottomrule
    \end{tabular}
    \label{tab:baselines}}
\end{table}

Table \ref{tab:baselines} shows that CPDs achieve competitive ED scores and substantially lower LSD and MMD scores, indicating better alignment with both large-scale structure and fine-scale variability than baseline methods. 
CPD can achieve this because they are initialized with calibration data, whose spatial structure is preserved under the flow.
Representative samples provided in Appendix \ref{apx:sec:qualitative}.

\begin{wrapfigure}[12]{r}{0.45\columnwidth}
    \vspace{-1.2em}
    \includegraphics[width=0.45\textwidth]{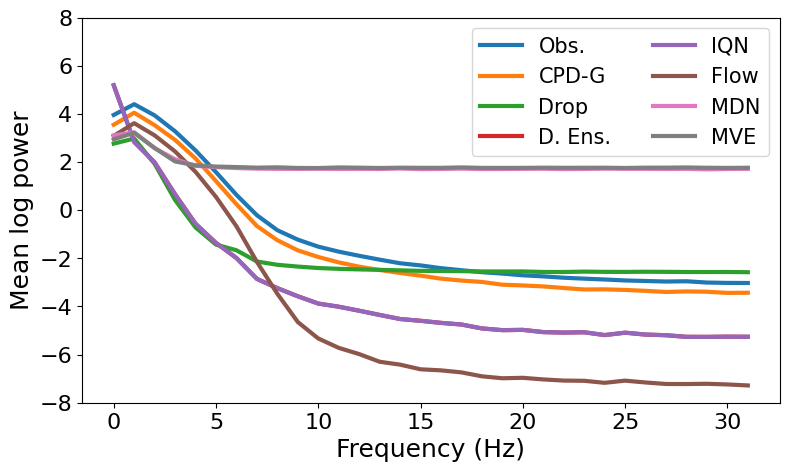}
  \caption{CPDs match spectral shape.}
\label{fig:spectral_decay}
\end{wrapfigure}
Figure \ref{fig:spectral_decay} shows the average log power spectrum of each method on the 2D GP regression task. Only CPD-G predictions match the target process across the full spectrum. Alternatives either over- or under-estimate high frequency behavior, leading to unrealistic and off-manifold samples (Figure \ref{fig:climate_samples}). Notably, even when the model predictions fail to capture high-frequency behavior (D. Ens), CPD-G still captures it.
%
%
Figure \ref{fig:climate_samples} shows example draws from the precipitation downscaling experiment under the CPD against Dropout and Flow. CPD samples avoid both oversmoothing and high-frequency noise injection, as reflected in the lower LSD and MMD metrics (Table \ref{tab:baselines}).
\begin{figure}[h]
\centering
\includegraphics[width=0.99\textwidth]{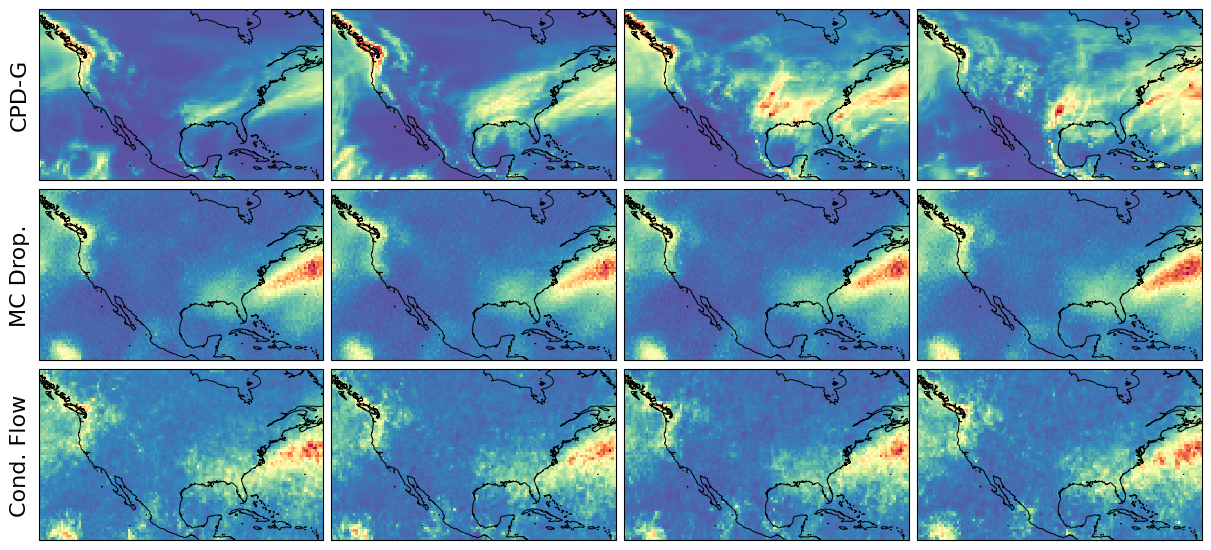}
\caption{Sample precipitation intensity from CPD-G, Dropout, and the conditional Normalizing Flow. CPDs avoid oversmoothing, excessive drizzle, and exhibit realistic heterogeneity across samples.}
\label{fig:climate_samples}
\end{figure}


\paragraph{Scenario Analysis and Targeted Sampling} Finally, we demonstrate how CPDs enable targeted sampling at arbitrary confidence levels, including extreme-tail and rare-event regimes that are inaccessible to conventional ensemble or conformal approaches. The HURDAT2 dataset records hurricane trajectories from 1851 through 2025 \citep{landsea2013atlantic}. We train a 1D CNN model to predict the next 12 time steps of the trajectory, given its previous 24 steps and auxiliary information including wind speed and pressure (Appendix \ref{apx:sec:data}). We conformalize the predicted trajectories using a naive $\ell_2$ score and conditional geometric trajectory score (CGT) defined in Appendix \ref{apx:sec:traj_scores}.

Figure \ref{fig:hurr} draws samples from the $\ell_2$ CPD and the CGT CPD within three $\alpha$ ranges: $\alpha \in (0.9, 1.0)$ (central samples), $\alpha \in (0.5, 1.0)$ (typical predictions) and $\alpha \in (0.0, 0.1)$ (extremes). Under either score, the CPD samples become increasingly more extreme as we lower the $\alpha$ range. These figures highlight CPDs as controllable predictive distributions rather than fixed-resolution uncertainty summaries, and their potential for highly efficient extreme tail sampling and rare-event estimation.

%

\begin{figure}[h]
\centering
\includegraphics[width=0.99\textwidth]{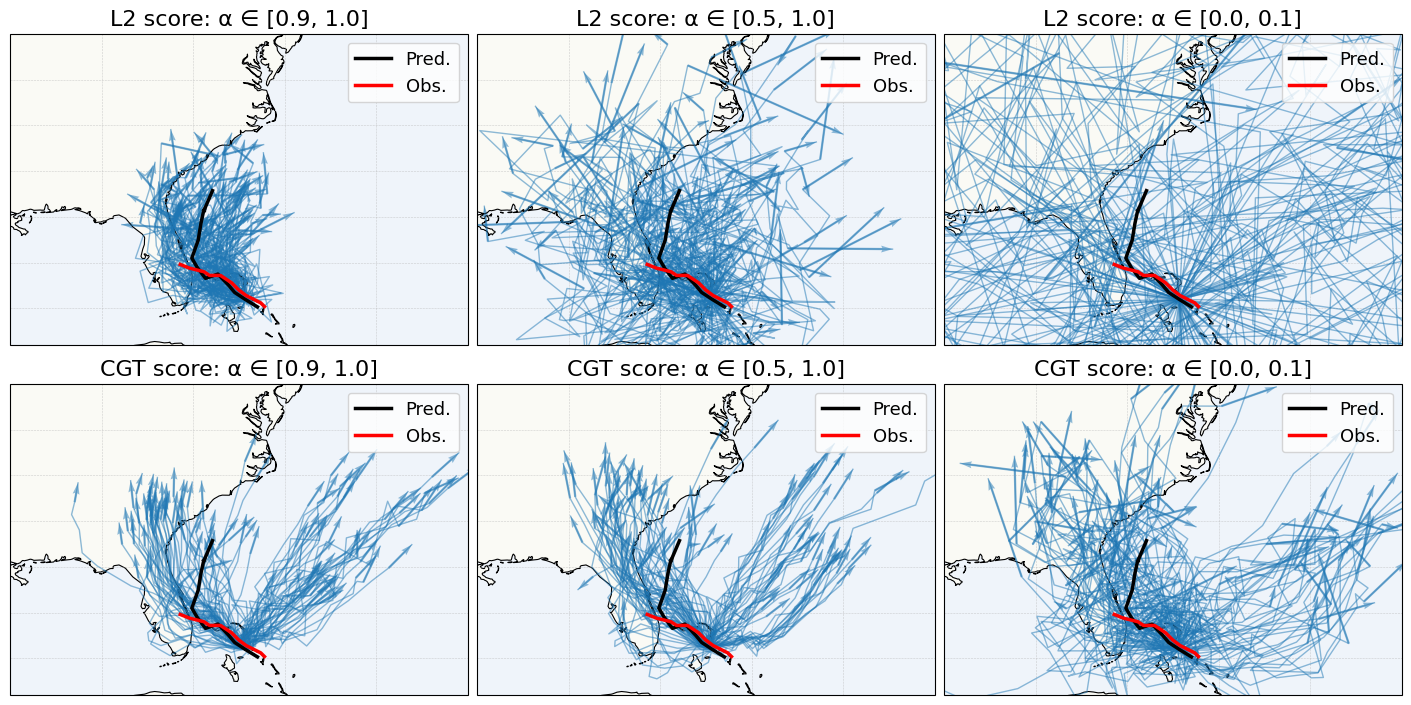}
\caption{CPDs can selectively emphasize any range of the distribution and their utility depends on the precision of the score. Scores that align with the data (CGT) produce more meaningful uncertainty.}
\label{fig:hurr}
\end{figure}

Furthermore, Figure \ref{fig:hurr} also highlights the importance of choosing an appropriate score. Assuming an $\ell_2$ score results in nonsensical trajectories as $\alpha \rightarrow 0$, while the CGT score is able to leverage conditional information (wind speed, pressure, location) and geometric constraints (velocity and curvature) to produce plausible samples at any range. The CGT score even shows a bifurcation pattern, which emerges purely from the score function, rather than the underlying model. 

%
%
%

\section{Discussion}

We introduced nonconformity flows as a constructive way to interact with differentiable conformal prediction sets by sampling their boundaries. This yields a scalable Monte Carlo representation of the conformal sets that requires no additional modeling, such as training generative models or learning transport maps. Moreover, we can mix across sequences of confidence levels to create conformal predictive distributions (CPDs).
 The resulting CPDs were shown to provide rigorous and highly competitive probabilistic forecasts (Section \ref{sec:experiments}).
 
There are several limitations of this work that we hope to address in future work. First, our approach does not guarantee full boundary exploration and, therefore, does not fully represent all possible samples consistent with the conformal level sets. This is generally desirable for forecasting, however, where we want to adhere closely to the data manifold, which is what the calibration data initialization helps us achieve. Additionally, the nonconformity flow requires a differentiable score. Non-smooth scores such as $\ell_1$ or $k$NN can cause the ODE solver to stall or fail to converge, as observed in Figure \ref{fig:convergence}. Proposition \ref{prop:w2_cpd}  is intended to clarify which components may influence forecasting skill, but it is not a tight bound. In future work we may study exact tangent diffusion for boundary exploration, applications to non-exchangeable conformal inference (e.g. temporally ordered data used in the precipitation and climate model debiasing experiments), and refinements on the bound.
 
Finally, while CPDs generate prediction samples, they are not proper generative models \textit{per se}. Integrating CPDs with generative models could potentially improve the performance of this approach; either via score learning or by coupling the nonconformity flow with a learned normalizing flow.
 
 

\bibliography{refs}

\appendix

\section{Nonconformity velocity field} \label{apx:sec:derivation}

\noindent  We recall the setting of Section~3.1. Fix a test input $x$ and prediction $\hat y$. Let
$S(\,\cdot\,,\hat y):\mathcal Y\to\mathbb R$ be continuously differentiable and write
$\nabla S(y) := \nabla S(y,\hat y)$. For a target miscoverage level $\alpha\in(0,1)$,
let $\tau_\alpha\in\mathbb R$ denote the split-conformal threshold, and define the
target level set
\[
\partial C_\alpha(x) \;:=\; \{y\in\mathcal Y : S(y,\hat y)=\tau_\alpha\}.
\]
Consider the controlled dynamical system (ODE)
\begin{align} 
   y'(t)&= v(y(t)) \\
   S'(y(t)) &= -\lambda(S(y(t)) - \tau_\alpha)
\end{align}
where $y(t) \in \mathcal{Y}$ is a trajectory in $\mathcal{Y}$ and $S(y(t)) = S(f_\theta(x), y(t))$ is the score of state $y(t)$ with $f_\theta(x) = \hat y \in \mathcal{Y}$ held fixed, and $v: \mathcal{Y} \times \mathbb{R} \to \mathcal{Y}$ is a vector field that governs the induced flow.

We show that the analytical solution $v_\alpha(\cdot)$ at any time $t \geq 0$ is 
\begin{equation}
    v_\alpha(y(t)) = -\lambda(S(y(t))- \tau_\alpha) \frac{\nabla S(y(t))}{\Vert \nabla S(y(t)) \Vert^2_2},
\end{equation}
where $\Vert \cdot \Vert_2^2$ is the squared Euclidean norm on $\mathcal{Y}$.

\paragraph{Derivation} First, we expand the score controller
$$
	S'(y(t)) = \nabla S(y(t))^T y'(t) = \nabla S(y(t))^T v(y(t)),
$$
to show that it is equivalent to the linear constraint:
\begin{equation} \label{eqn:linear_control}
	\nabla S(y(t))^T v(y(t)) = -\lambda(S(y(t)) - \tau_\alpha).
\end{equation}
Among all velocity fields $v(\cdot)$ satisfying \ref{eqn:linear_control}, we choose the minimum Euclidean norm velocity field
\begin{equation}
	v_\alpha(y(t)) \in \arg\min_{v} \frac{1}{2}||v||_2^2 \quad \text{s.t.} \quad \nabla S(y(t))^T v(y(t)) = -\lambda(S(y(t)) - \tau_\alpha).
\end{equation}
This is a single constraint quadratic program, which we rewrite in Lagrangian form as
\[
\mathcal{L}(v,\eta) = \frac12\|v\|_2^2 + \eta\left(\nabla S(y(t))^\top v(y(t)) + \lambda(S(y(t))-\tau_\alpha)\right).
\]
The first-order condition gives
\[
\nabla_v\mathcal L(v,\eta)=v+\eta\nabla S(y(t))=0, \qquad v=-\eta\nabla S(y(t)).
\]
Enforcing the constraint,
\[
-\eta\|\nabla S(y(t))\|_2^2 = -\lambda(S(y(t))-\tau_\alpha),
\]
so, provided $\|\nabla S(y(t))\|_2\neq 0$,
\[
\eta = \lambda\frac{S(y(t))-\tau_\alpha}{\|\nabla S(y(t))\|_2^2}.
\]
Substituting into $v=-\eta\nabla S(y(t))$ yields
\[
v_\alpha(y(t)) = -\lambda(S(y(t))-\tau_\alpha) \frac{\nabla S(y(t))}{\|\nabla S(y(t))\|_2^2}.
\]

\section{Proofs} \label{apx:sec:proofs}

\noindent  We recall the setting of Section~3.1.
Fix $x \in \mathcal{X}$ and define the score field $S:\mathcal{Y}\to\mathbb{R}$ by $S(y) := S(x,y)$. Let $\tau_\alpha \in \mathbb{R}$ denote the split-conformal threshold and define the conformal boundary
\[
\partial C_\alpha(x) \;:=\; \{y \in \mathcal{Y} : S(y) = \tau_\alpha\}.
\]
Let $y(t)=\Phi_\alpha(t,y_0)$ denote the flow induced by the ODE
\begin{equation}\label{eq:ode_flow}
	y'(t) \;=\; v_\alpha(y(t)), \qquad v_\alpha(y) \;:=\; -\lambda\,\frac{S(y)-\tau_\alpha}{\|\nabla S(y)\|_2^2}\,\nabla S(y),
\end{equation}
whenever $\nabla S(y)\neq 0$, where $\lambda>0$ is fixed. 

\subsection{Proof of Proposition~\ref{prop:convergence}}
We restate Proposition~\ref{prop:convergence} here: 
\begin{proposition}[Convergence]
Define the score error $\varepsilon(t):=S(y(t))-\tau_\alpha$. Assume $y(t)$ exists for all $t \geq 0$ and $\nabla S(y(t)) \neq 0$ whenever $S(y(t)) \neq \tau_\alpha$. Then:
\begin{enumerate}
	\item \textbf{Score convergence:}
	$\varepsilon'(t)=-\lambda\,\varepsilon(t)$, hence
	\[
	\varepsilon(t)=\varepsilon(0)\,e^{-\lambda t}
	\quad\text{and}\quad
	S(y(t))\to\tau_\alpha \ \text{as } t\to\infty.
	\]
	In particular, every $\omega$-limit point of $(y(t))_{t\ge 0}$ lies in $\partial C_\alpha(x)$.

	\item \textbf{Pointwise convergence:}
	Assume, in addition, that there exist $T<\infty$ and $m>0$ such that
	\[
		\|\nabla S(y(t))\|_2 \ge m
		\qquad \text{for all } t\ge T.
	\]
	Then $y(t)$ converges to a unique limit point $y_\infty \in \partial C_\alpha(x)$ and
	\[
		\|y(t)-y_\infty\|_2
		\le
		\frac{1}{m}\,|S(y_0)-\tau_\alpha|\,e^{-\lambda t}
		\qquad\text{for all } t\ge T.
	\]
\end{enumerate}
\end{proposition}

\begin{proof}
\textbf{Step 1 (score convergence).}
By the chain rule,
\[ \varepsilon'(t) = S'(y(t)) = \nabla S(y(t))^\top y'(t) 
= \nabla S(y(t))^\top v_\alpha(y(t)).
\]
Substituting the definition of $v_\alpha$ from \eqref{eq:ode_flow} yields
\[
\varepsilon'(t) = \nabla S(y(t))^\top 
\left( 
-\lambda\,\frac{S(y(t))-\tau_\alpha}{\|\nabla S(y(t))\|_2^2}\,\nabla S(y(t))
\right)
= -\lambda\,(S(y(t))-\tau_\alpha)
= -\lambda\,\varepsilon(t).
\]
Therefore $\varepsilon(t)=\varepsilon(0)e^{-\lambda t}$, which implies $S(y(t))\to\tau_\alpha$ as $t\to\infty$.

\textbf{Step 2 ($\omega$-limit points).}
Let $\bar y$ be any $\omega$-limit point of $(y(t))_{t\ge 0}$, i.e., there exists a sequence $t_k\to\infty$ such that $y(t_k)\to \bar y$. By continuity of $S$ and Step~1,
\[
S(\bar y) = \lim_{k\to\infty} S(y(t_k))
= \tau_\alpha,
\]
so $\bar y \in \partial C_\alpha(x)$.

\textbf{Step 3 (pointwise convergence).}
Assume there exist $T<\infty$ and $m>0$ such that $\|\nabla S(y(t))\|_2\ge m$ for all $t\ge T$. For $t\ge T$ we have
\[
\|y'(t)\|_2 = \|v_\alpha(y(t))\|_2
= \lambda\,\frac{|S(y(t))-\tau_\alpha|}{\|\nabla S(y(t))\|_2}
\le \frac{\lambda}{m}\,|S(y(t))-\tau_\alpha|.
\]
Using Step~1 again gives
\[
|S(y(t))-\tau_\alpha|
= |S(y_0)-\tau_\alpha|e^{-\lambda t},
\]
and hence
\[
\|y'(t)\|_2
\le \frac{\lambda}{m}\,|S(y_0)-\tau_\alpha|\,e^{-\lambda t},
\qquad t\ge T.
\]
It follows that
\[
\int_T^\infty \|y'(t)\|_2\,dt < \infty,
\]
so $(y(t))_{t\ge T}$ has finite arc length and is Cauchy. Therefore $y(t)$ converges to a limit $y_\infty\in\mathcal Y$. By Step~2, every $\omega$-limit point lies in $\partial C_\alpha(x)$, so necessarily $y_\infty\in\partial C_\alpha(x)$. Finally, for any $t\ge T$,
\[
\|y_\infty-y(t)\|_2
\le \int_t^\infty \|y'(u)\|_2\,du
\le \int_t^\infty
\frac{\lambda}{m}\,|S(y_0)-\tau_\alpha|\,e^{-\lambda u}\,du
=
\frac{1}{m}\,|S(y_0)-\tau_\alpha|\,e^{-\lambda t},
\]
which is the claimed rate bound.
\end{proof}

\subsection{Proof of Proposition~\ref{prop:cpd-calibration}}

We restate Proposition~\ref{prop:cpd-calibration} here: 
\begin{proposition}[Calibration]
Let $P_x^{\mathrm{CPD}}$ denote a CPD constructed from boundary measures
$\{\nu_{x, \alpha}\}_{\alpha \in (0,1)}$ and mixing measure $\pi$. Then, for every
$\alpha \in (0,1)$,
$ P_x^{\mathrm{CPD}}(C_\alpha(x)) = \pi\{\beta : \tau_\beta \le \tau_\alpha\} \ge \pi([\alpha,1)). $
Moreover, for every empirical conformal level
$\alpha_k\in\mathcal A_{\mathrm{conf}}$,
\[
P_x^{\mathrm{CPD}}(C_{\alpha_k}(x)) = \pi([\alpha_k,1)).
\]
In particular, if $\pi=\mathrm{Unif}(0,1)$, then
$P_x^{\mathrm{CPD}}(C_\alpha(x))\ge 1-\alpha$ for all $\alpha$, with equality on the empirical conformal grid $\mathcal A_{\mathrm{conf}}$.
\end{proposition}

\begin{proof}
Let $\nu_{x, \beta}$ denote a boundary measure supported on $\partial C_\beta(x)$. Any draw $Y\sim\nu_{x, \beta}$ will satisfy $S(x,Y)=\tau_\beta$ almost surely by definition. Therefore, the measure of any set $C_\alpha(x)$ under $\nu_{x, \beta}$ is given by
\[
\nu_{x,\beta}(C_\alpha(x)) = \mathbf 1\{\tau_\beta \le\tau_\alpha\}.
\]
Integrating over $\beta \sim\pi$ gives
\[
P_x^{\mathrm{CPD}}(C_\alpha(x))
=
\int_0^1 \nu_{x,\beta}(C_\alpha(x)) d\pi(\beta)
=
\int_0^1 \mathbf 1\{\tau_\beta \le\tau_\alpha\}\,d\pi(\beta)
=
\pi\{\beta :\tau_\beta \le \tau_\alpha\}.
\]
Now, because $\tau_\beta$ is a nonincreasing function of $\beta$, we get that $\beta \ge \alpha$ implies $\tau_\beta \le \tau_\alpha$, and hence
\[
\pi\{\beta:\tau_\beta \le \tau_\alpha\}\ge \pi([\alpha,1)).
\]

At empirical conformal levels $\alpha_k\in\mathcal A_{\mathrm{conf}}$, the threshold changes exactly at $\alpha_k$, so
\[
\{\beta:\tau_\beta \le \tau_{\alpha_k}\} = [\alpha_k,1).
\]
up to endpoints assuming no tied scores, which trivially holds for any continuous score almost surely. Therefore, at the empirical conformal level sets we have exact $\pi$-calibration:
\[
P_x^{\mathrm{CPD}}(C_{\alpha_k}(x))
=
\pi([\alpha_k,1)).
\]
The uniform case follows immediately.
\end{proof}

\subsection{Proof of Proposition~\ref{prop:w2_cpd}} \label{proof:w2_cpd}

We restate Proposition~\ref{prop:w2_cpd} here:

\begin{proposition}[$W_2$ approximation bound]
Fix $x\in\mathcal X$ and let $P_x^{\mathrm{proj}} = \int_0^1(\mathrm{proj}_{x,\alpha})_\# P_x^\star \,d\alpha $ denote the distribution of $\mathrm{proj}_{x,\alpha}(Y)$ for $Y\sim P_x^\star = P(\cdot \mid x)$ and $\alpha\sim\mathrm{Unif}(0,1)$. Under the regularity conditions
stated in Appendix~\ref{proof:w2_cpd}, 
\[
W_2(P_x^\star, P_x^{\mathrm{CPD}})
\le
W_2(P_x^\star, P_x^{\mathrm{proj}})
+
(\mathbb E_\alpha L_\alpha^2)^{1/2}W_2(P_x^\star, \mu_x)
+
C\{\mathbb E_\alpha[\bar\kappa(\alpha)^2R_4(\alpha)]\}^{1/2},
\]
where $R_4(\alpha):=\mathbb E_{\mu_x}[\mathrm{dist}(y,\partial C_\alpha(x))^4]$,
$\bar\kappa(\alpha)$ controls the curvature of the score-gradient paths, and
$C>0$ is a geometric constant.
\end{proposition}

\paragraph{Regularity conditions} Let $\mathrm{proj}_{x,\alpha}$ denote nearest-point projection operator onto $\partial C_\alpha(x)$. We will assume that projection onto $\partial C_\alpha(x)$ is unique $P_x^\star$- and $\mu_x$-almost surely and  that $\mathrm{proj}_{x,\alpha}$ is $L_\alpha$-Lipschitz on $\mathrm{supp}(P_x^\star)\cup\mathrm{supp}(\mu_x)$ up to sets of measure zero with $\mathbb{E}_\alpha L_\alpha^2 < \infty$.
We also assume that the score $S(x, \cdot) \in C^2$ and that the norm of its gradient and Hessian are bounded
\[
0<c\le \|\nabla_z S(x,z)\|_2\le C_S<\infty,
\qquad
\|\nabla_z^2S(x,z)\|_{\mathrm{op}}\le H_\alpha<\infty
\]
along the flow paths from $y_0\sim\mu_x$ to $\partial C_\alpha(x)$ and the projection segments from $y_0$ to $\mathrm{proj}_{x,\alpha}(y_0)$. Finally, if $q_{x,\alpha}(y_0)$ denotes the first intersection of the initial score gradient ray from $y_0$ with $\partial C_\alpha(x)$, we assume
\[
\kappa_{\mathrm{proj}}(\alpha)
:= \operatorname*{ess\,sup}_{y_0\sim\mu_x} 
\frac{\|q_{x,\alpha}(y_0)-\mathrm{proj}_{x,\alpha}(y_0)\|_2}
{\mathrm{dist}(y_0,\partial C_\alpha(x))^2}
< \infty,
\]
with $\kappa_{\mathrm{proj}}(\alpha) = 0$ when $\mathrm{dist}(y_0,\partial C_\alpha(x))=0$. In the proof we use $\bar\kappa(\alpha):=H_\alpha/c+\kappa_{\mathrm{proj}}(\alpha)$.


\begin{proof}
First, we define the pushforward distribution of the base measure $\mu_x$ under the projection operator $\mathrm{proj}_{x,\alpha}$ as
\[
P_x^{\mathrm{proj},\mu}:=\int_0^1(\mathrm{proj}_{x,\alpha})_\#\mu_x\,d\alpha .
\]
Then, by the triangle inequality we can decompose the prediction error $W_2(P_x^\star, P_x^{\mathrm{CPD}})$ as
\[
W_2(P_x^\star, P_x^{\mathrm{CPD}})
\le
W_2(P_x^\star, P_x^{\mathrm{proj}})
+
W_2(P_x^{\mathrm{proj}},P_x^{\mathrm{proj},\mu})
+
W_2(P_x^{\mathrm{proj},\mu},P_x^{\mathrm{CPD}}).
\]

\textbf{Term 1}: The first term, which we leave unchanged, is the cost of transporting the true conditional $P_x^\star$ to its projection onto the CPD class, denoted $P_x^{\mathrm{proj}}$. This is the unavoidable cost of imposing a score $S$ on the density of the data.

\textbf{Term 2}: The second term is, essentially, the cost of projecting onto the conformal filtration with an arbitrary base measure $\mu_x$, instead of the optimal base measure $P_x^\star$. We can further bound this term via the discrepancy between $\mu_x$ and $P_x^\star$ and the regularity of the level sets.

Let $\gamma$ be an optimal coupling of $P_x^\star$ and $\mu_x$. Drawing $(Y,Z)\sim\gamma$ and $\alpha\sim\mathrm{Unif}(0,1)$ independently gives the projection coupling
\[
(\mathrm{proj}_{x,\alpha}(Y),\mathrm{proj}_{x,\alpha}(Z))
\]
of $P_x^{\mathrm{proj}}$ and $P_x^{\mathrm{proj},\mu}$. Hence, because $\mathrm{proj}_{x,\alpha}$ is $L_\alpha$-Lipschitz on $\mathrm{supp}(P_x^\star) \cup \mathrm{supp}(\mu_x)$, we get
\[
W_2^2(P_x^{\mathrm{proj}},P_x^{\mathrm{proj},\mu})
\le
\int_0^1
\mathbb E_\gamma
\|\mathrm{proj}_{x,\alpha}(Y)-\mathrm{proj}_{x,\alpha}(Z)\|_2^2\,d\alpha
\le
\left(\int_0^1L_\alpha^2\,d\alpha\right)W_2^2(P_x^\star, \mu_x).
\]
since $\gamma$ is the optimal coupling between $P_x^\star$ and $\mu_x$. 

\textbf{Term 3}: The final term is the cost of projecting onto the conformal filtration using the nonconformity flow rather than the nearest-point projection operator $\mathrm{proj}_{x,\alpha}$. This term is controlled by two geometric effects: the bending of the score-gradient path and the mismatch between the initial score-gradient ray and the nearest-projection ray.

Let $y_0\sim\mu_x$ be a base-measure sample and let $\eta_{y_0,\alpha}$ denote the flow path from $y_0$ to $\Pi_{x,\alpha}(y_0)$. We define $p=\mathrm{proj}_{x,\alpha}(y_0)$ as the nearest-point projection onto $\partial C_\alpha(x)$ and define $d_\alpha(y_0)=\|y_0-p\|_2$ as the length of the projection segment. We also define the normalized score-gradient field
\[
n_x(z):=\frac{\nabla_z S(x,z)}{\|\nabla_z S(x,z)\|_2}.
\]
Let $\sigma\in\{-1,1\}$ denote the sign of the flow direction, so that the flow moves along $\sigma n_x(\cdot)$ toward the target level set. Let $q_{x,\alpha}(y_0)$ denote the first intersection of the initial score-gradient ray
\[
\{y_0+r\sigma n_x(y_0):r\ge 0\}
\]
with $\partial C_\alpha(x)$. Thus $q_{x,\alpha}(y_0)$ is the boundary point obtained by following the initial score-gradient direction without allowing the direction to bend.

We write
\[
\bar\kappa(\alpha):=\kappa_{\mathrm{flow}}(\alpha)+\kappa_{\mathrm{proj}}(\alpha),
\]
where $\kappa_{\mathrm{flow}}(\alpha)$ controls the curvature of the normalized score-gradient field along the flow paths and $\kappa_{\mathrm{proj}}(\alpha)$ controls the mismatch between the initial score-gradient ray and the nearest-projection ray. Specifically, these are defined as 
\[
\kappa_{\mathrm{flow}}(\alpha)
:=
\operatorname*{ess\,sup}_{y_0\sim\mu_x}
\sup_{z\in\eta_{y_0,\alpha}}
\left\|
D n_x(z)
\right\|_{\mathrm{op}},
\]
where $D$ denotes the Jacobian with respect to $z$, and from the regularity conditions
\[
\kappa_{\mathrm{proj}}(\alpha)
:=
\operatorname*{ess\,sup}_{y_0\sim\mu_x}
\frac{\|q_{x,\alpha}(y_0)-\mathrm{proj}_{x,\alpha}(y_0)\|_2}{d_\alpha(y_0)^2},
\]
with the ratio taken to be zero when $d_\alpha(y_0)=0$. This separates the curvature of the flow path from the error induced by using the initial score-gradient ray rather than exact nearest-point projection. $\kappa_{\mathrm{flow}}(\alpha)$ is finite by the assumption that $\|\nabla_z^2S(x,z)\|_{\mathrm{op}}\le H_\alpha<\infty$ and $\kappa_{\mathrm{proj}}(\alpha)$ is finite by the regularity condition.

We first show that the arclength of the flow path $\eta_{y_0,\alpha}$, denoted $\ell_{y_0,\alpha}$, is finite and controlled by $d_\alpha$. Let $y(t)$ denote the flow trajectory initialized at $y(0)=y_0$. By the same argument as in Proposition~\ref{prop:convergence}, the lower gradient bound along the flow path, $c \le \|\nabla_z S(x,z)\|_2$, gives
\[
\ell_{y_0,\alpha}
\le
\frac{1}{c}|S(x,y_0)-\tau_\alpha|.
\]
Since $p=\mathrm{proj}_{x,\alpha}(y_0)$ is a projection onto $\partial C_\alpha(x)$, then by definition we have that $S(x,p)=\tau_\alpha$. Therefore, because the upper gradient bound, $\|\nabla_z S(x,z)\|_2\le C_S$ holds on the projection segment from $y_0$ to $p$, we get by mean value theorem that 
\[
|S(x,y_0)-\tau_\alpha|
=
|S(x,y_0)-S(x,p)|
\le
C_S\|y_0-p\|_2
=
C_Sd_\alpha(y_0).
\]
Thus the arclength of the flow path $\eta_{y_0,\alpha}$ is finite and bounded as 
\[
\ell_{y_0,\alpha}\le (C_S/c)d_\alpha(y_0).
\]

Given a finite and well controlled path length, we can meaningfully bound the distance between $\Pi_{x,\alpha}(y_0)$ and $p=\mathrm{proj}_{x,\alpha}(y_0)$. We insert the initial-ray endpoint $q_{x,\alpha}(y_0)$ and apply the triangle inequality to get 
\[
\|\Pi_{x,\alpha}(y_0)-p\|_2
\le
\|\Pi_{x,\alpha}(y_0)-q_{x,\alpha}(y_0)\|_2
+
\|q_{x,\alpha}(y_0)-p\|_2.
\]
For the first term, we parametrize $\eta_{y_0,\alpha}$ by arclength as $\eta(s)$, $0\le s\le\ell_{y_0,\alpha}$. Because the sign of $S(x,y(t))-\tau_\alpha$ is fixed until the path reaches the target level set, we can express the gradient with respect to the arclength as
\[
\eta'(s)=\sigma n_x(\eta(s)).
\]
Differentiating with respect to arclength gives
\[
\eta''(s)=\sigma Dn_x(\eta(s))\eta'(s),
\]
where $D$ is, again, the Jacobian operator. By the induced operator norm inequality and the definition of $\kappa_{\mathrm{flow}}(\alpha)$,
\[
\|\eta''(s)\|_2
\le
\|Dn_x(\eta(s))\|_{\mathrm{op}}\|\eta'(s)\|_2
\le
\kappa_{\mathrm{flow}}(\alpha),
\]
since $\|\eta'(s)\|_2=1$ by definition. Therefore,
\[
\|\eta'(s)-\eta'(0)\|_2
\le
\int_0^s\|\eta''(r)\|_2\,dr
\le
\kappa_{\mathrm{flow}}(\alpha)s.
\]
This means that, after arclength $s$, the flow direction can differ from the initial score-gradient direction by at most $\kappa_{\mathrm{flow}}(\alpha)s$. Integrating this directional deviation along the path gives
\[
\|\Pi_{x,\alpha}(y_0)-q_{x,\alpha}(y_0)\|_2
\le
C\int_0^{\ell_{y_0,\alpha}}\kappa_{\mathrm{flow}}(\alpha)s\,ds
\le
C\kappa_{\mathrm{flow}}(\alpha)\ell_{y_0,\alpha}^2.
\]
The second term is controlled directly by the definition of $\kappa_{\mathrm{proj}}(\alpha)$:
\[
\|q_{x,\alpha}(y_0)-p\|_2
\le
\kappa_{\mathrm{proj}}(\alpha)d_\alpha(y_0)^2.
\]
Combining the two bounds and using $\ell_{y_0,\alpha}\le (C_S/c)d_\alpha(y_0)$ gives
\[
\|\Pi_{x,\alpha}(y_0)-\mathrm{proj}_{x,\alpha}(y_0)\|_2
\le
C\{\kappa_{\mathrm{flow}}(\alpha)+\kappa_{\mathrm{proj}}(\alpha)\}d_\alpha(y_0)^2
=
C\bar\kappa(\alpha)d_\alpha(y_0)^2,
\]
where $C$ absorbs the regularity constants, including $C_S/c$.

Finally, to bound $W_2^2(P_x^{\mathrm{proj},\mu},P_x^{\mathrm{CPD}})$, sample $Y\sim\mu_x$ and $\alpha\sim\mathrm{Unif}(0,1)$ independently so that 
$
(\mathrm{proj}_{x,\alpha}(Y),\Pi_{x,\alpha}(Y))
$
is a coupling of $P_x^{\mathrm{proj},\mu}$ and $P_x^{\mathrm{CPD}}$.  Since $W_2$ is defined by the optimal coupling, evaluating it under this particular coupling gives an upper bound
\[
W_2^2(P_x^{\mathrm{proj},\mu},P_x^{\mathrm{CPD}})
\le
\int_0^1
\mathbb E_{\mu_x}
\|\mathrm{proj}_{x,\alpha}(Y)-\Pi_{x,\alpha}(Y)\|_2^2\,d\alpha
\le
C^2
\int_0^1
\bar\kappa(\alpha)^2R_4(\alpha)\,d\alpha .
\]
Taking square roots gives the bound for the third term.
\end{proof}

\section{Risk--controlling Prediction Bands} \label{sec:risk_control}

Sampling from conformal prediction set boundaries provides a geometric representation of admissible predictions at a fixed confidence level. However, boundary samples obtained via the nonconformity flow do not yield pointwise calibrated prediction bands, i.e. coordinate-wise envelopes of these samples do not necessarily achieve pointwise coverage. We show how pointwise risk-controlling prediction bands can be obtained from boundary samples via a reconformalization step.

Using the nonconformity flow $\Phi(\cdot)$ (Equation \ref{eqn:flow}), we generate samples $y^{(1)}, \ldots, y^{(M)}$ lying on (or arbitrarily close to) the boundary set $\partial C_\alpha(x) = \{y \in \mathcal{Y} : S(x, y) = \tau_\alpha \}$. From these samples, we form a preliminary coordinate-wise envelope $[\ell_0(x),u_0(x)]$ by taking pointwise empirical quantiles (e.g., $\alpha/2$ and $1-\alpha/2$) across the sampled functions. This envelope reflects the geometry of the conformal set but does not, by itself, satisfy finite-sample pointwise coverage, or even pointwise risk control \citep{bates2021distribution, ma2024calibrated}.

To obtain a prediction band with pointwise risk control, we introduce a scalar inflation parameter $\eta \geq 0$ and define the prediction band $\mathcal{B}_\eta(x) =\bigl\{y:\ell_0(x)-\eta \le y \le u_0(x)+\eta \bigr\}$. We evaluate pointwise coverage through the risk functional 
\[
\ell_\eta(x,y) =\frac{1}{p}\sum_{j=1}^p \mathbf{1}\!\left\{y_j \notin \bigl[\ell_0(x)_j-\eta,\ u_0(x)_j+\eta\bigr]\right\},
\]
which measures the fraction of $p$ discretization points outside the band. For a prescribed tolerance $\delta \in (0,1)$ and confidence level $1-\alpha$, we say $\mathcal{B}_\eta$ controls risk at level $(\delta, \alpha)$ if
\begin{equation}\label{eqn:pointwise_coverage}
P(\ell_\eta(X,Y)\le \delta \bigr)\ge 1-\alpha.
\end{equation}
Since $\ell_\eta(x,y)$ is nonincreasing in $\eta$, we set $\eta$ as the smallest value achieving $(\delta, \alpha)$-risk control (Equation \ref{eqn:pointwise_coverage}) on held-out calibration data; in all experiments we use $\delta = \alpha$. Because the data are exchangeable, this yields a prediction band $\mathcal{B}_\eta$ with finite-sample risk control \citep{bates2021distribution} as long as we use a second calibration split to compute $\eta$. If calibration data is highly limited, we note that the dependence between $\tau_\alpha$ and $\eta$ is $O(1/n)$ so the calibration can be re-used with little practical impact on the test coverage (Table \ref{tab:risk_control}), although the strict guarantee will no longer hold.
Asymmetric reconformalization is also possible by setting the upper and lower inflation term separately.

\paragraph{Risk Controlling Bands}
\begin{table}
    \caption{\textbf{Top:} marginal coverage ($\alpha = 0.1$). \textbf{Bottom:} median prediction set width.}
    \centering
    \begin{tabular}{lccccc}
        \toprule
              & Symm. & Asym. & $\Delta \mu$ & $\Delta (\mu, \sigma)$ \\
        \midrule
         Sample & 0.894 & 0.646 & 0.828 & 0.876 \\
         Re-conf & 0.901 & 0.900 & 0.900 & 0.901 \\
         RCPS & 0.900 & 0.901 & 0.901 & 0.902 \\
         UQNO & 0.901 & 0.902 & 0.906 & 0.907 \\
         \midrule
         Sample & 2.443 & 2.032 & 1.145 & 4.269 \\
         Re-conf & 2.479 & 2.948 & 1.302 & 4.465 \\
         RCPS & 2.460 & 3.487 & 2.451 & 5.715 \\
         UQNO & 2.467 & 2.982 & 2.348 & 5.772 \\
        \bottomrule
    \end{tabular}
    \label{tab:risk_control}
\end{table}
Table \ref{tab:risk_control} shows the marginal coverage ($\alpha = 0.1$) and median prediction set width under four 1D functional regression settings Appendix \ref{apx:sec:data}. In the symmetric case, all methods perform the same. For asymmetric and biased data, both mean biased ($\Delta \mu$) and mean+variance biased ($\Delta (\mu, \sigma)$), the nonconformity flow adapts while alternative methods, RCPS \citep{bates2021distribution} and UQNO \citep{ma2024calibrated}, do not. This is due to our approach conformalizing in the native space of the data, rather than their graphs, thus better reflecting its geometric properties. Representative bands are shown in Appendix \ref{apx:sec:simulations}.

\section{Tangent Repulsion for Boundary Exploration}
\label{sec:tangent-repulsion}


\begin{wrapfigure}[14]{r}{0.4\columnwidth}
    \vspace{-1.2em}
    \centering
    \includegraphics[width=0.4\textwidth]{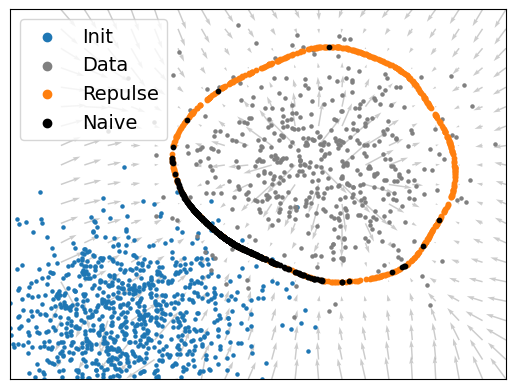}
  \caption{Repulsive boundary exploration under a $k$NN score (Section \ref{apx:sec:vec_scores}).}
\label{fig:repulsion}
\end{wrapfigure}
Because the flow $\Phi_\alpha(t,y_0)$ is everywhere co-linear with $\nabla S$, the induced boundary measures $\{\nu_{x,\alpha}\}_{\alpha\in(0,1)}$ are, in effect, projections of the base measure $\mu_x$ onto the level sets $\{\partial C\alpha(x)\}_{\alpha\in(0,1)}$ along integral curves of $\nabla S$. Consequently, not all boundary points are reachable from a given base measure: trajectories are attracted to nearby regions of the boundary and may never explore distant regions (Figure \ref{fig:repulsion}).
This behavior is partially desirable. For instance, when the base measure $\mu_x$ is concentrated near the data manifold (e.g., calibration data), then $P_x^{\mathrm{CPD}}$ is also close to the manifold. However, this also induces a systematic bias:  trajectories tend to accumulate in geometrically flat, nearby regions of the boundary, leading to poor exploration of $\partial C_\alpha(x)$.

To mitigate this effect, we introduce a deterministic tangent repulsion velocity field that spreads trajectories along the level set while preserving the score constraint. The minimum-norm controller used to drive trajectories to the score level set $\partial C_\alpha(x)$ enforces a prescribed scalar ODE for $S(y(t))$, but does not uniquely specify the velocity field. In particular, any velocity of the form
\begin{equation}
y'(t) = v_{\mathrm{n}}(y(t)) + v_{\mathrm{t}}(y(t)), \qquad \nabla S(y)^\top v_{\mathrm{t}}(y)=0,
\label{eq:tangent-decomp}
\end{equation}
induces the same evolution of $S(y(t))$ as the minimum-norm flow (Equation \ref{eqn:frontier_velocity}), since $v_{\mathrm{t}}$ lies in the tangent space of $\partial C_\alpha(x)$. 
  \begin{algorithm}[H]
   \caption{Tangent Repulsion}
   \label{alg:repulsion}
   \begin{algorithmic}
   \STATE \textbf{Input:} Boundary points $\{y_i(t)\}_{i=1}^B$, step size $\delta$
   \FOR{$k=1$ to $K$}
      \STATE 1. Compute $R_i(t) = \sum_{j\neq i}(y_i-y_j)/(\|y_i-y_j\|_2^2 + \epsilon)$ and $g_i = \nabla S(y_i)$
      \STATE 2. Compute $v_{\mathrm{t}}(y_i) = (I-g_i g_i^\top/\|g_i\|^2)R_i$
      \STATE 3. Step $y_i \leftarrow y_i + \delta\, v_{\mathrm{t}}(y_i)$
      \STATE 4. Correct $y_i \leftarrow \lim_{t \rightarrow \infty} \Phi_\alpha(t, y_i)$
   \ENDFOR
   \STATE \textbf{Return:} Equalized boundary points $\{y_i(t)\}_{i=1}^B$
   \end{algorithmic}
  \end{algorithm}

\section{Additional Simulation} \label{apx:sec:simulations}

Here we provide a few additional simulations to support some of our claims in the manuscript. These include computational scaling which compares the runtime of our CPD approach against the baseline models, a demonstration of sample diversity at fixed alpha levels, how to convert our conformal prediction sets into risk controlling bands, and the effects of repulsion sampling.

\subsection{Computational Scaling} \label{apx:sec:scaling}

We measure inference time to generate 100 predictive samples on 1D Gaussian process regression tasks (Appendix \ref{apx:sec:data}) as the output dimension $p$ varies from 32 to 2048. CPD uses a global $\ell_2$ score with 20 forward Euler steps. All neural baselines use 1D convolutional architectures with width 64 and depth 4. Models are not trained since inference cost is independent of the parameter values. 

\begin{figure}[h]
\centering
\includegraphics[width=\textwidth]{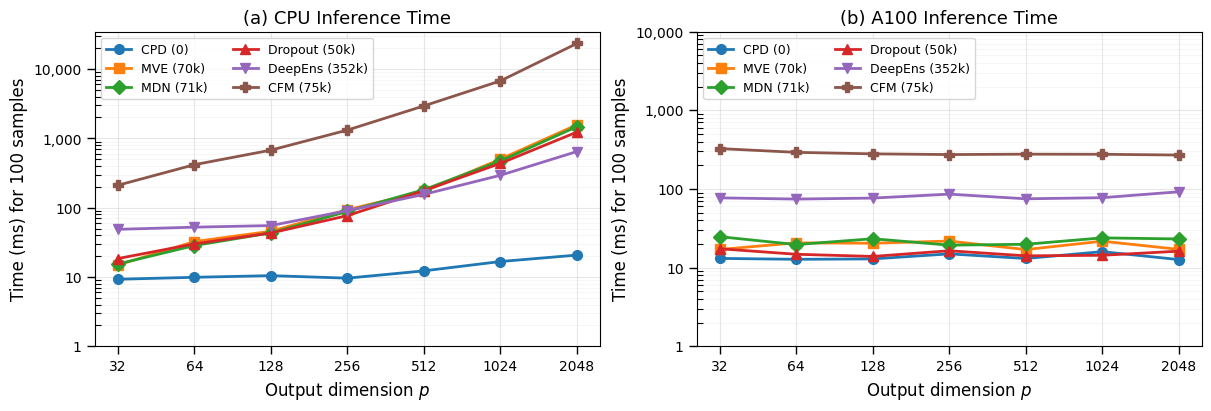}
\caption{Inference time for each baseline method and the CPD using a global $\ell_2$ score.}
\label{fig:runtime}
\end{figure}

Figure \ref{fig:runtime} shows the log growth of inference time for each method to draw 100 samples on either a standard CPU (panel a) or an A100 GPU (panel b). CPU inference time demonstrates that, as claimed, sampling from the CPD scales very well with dimension. In fact, its run time is nearly constant until $p \geq 512$. Other methods have comparable runtimes at $p = 32$, but scale dramatically worse than the CPD. This is because the CPD does not require any additional forward passes from the underlying model and only needs to evaluate score gradients. More complex scores may change the constant, but will not change the scaling rate as long as they are based on distance. The second panel shows runtime on an A100 GPU. Here the GPU is large enough to run each method fully in parallel for 100 samples. In this case, all methods have constant runtime, however CPD has the lowest average.

\subsection{Admissible Representations}

\begin{wrapfigure}[15]{r}{0.45\columnwidth}
    \includegraphics[width=0.45\textwidth]{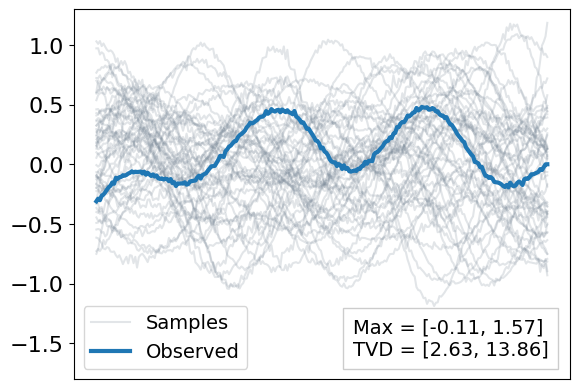}
  \caption{Sample realizations from a fixed $\alpha$-level showing considerable heterogeneity.}
\label{fig:diversity}
\end{wrapfigure}
Figure \ref{fig:diversity} shows multiple realizations from the $\alpha = 0.1$ level set of an $\ell_2$ score applied to 1D functional regression task. There is significant heterogeneity within these samples, showing that, even for high-dimensional structured regression problems, the flow sampler does not collapse to a degenerate mode. In fact, these samples exhibit quantitatively and qualitatively different behavior compared to the original $(1 - \alpha) \times 100\%$ most outlying function. Hence, sampling exposes the geometry and internal variability of conformal prediction sets, which is invisible to classical set representations. 

These samples characterize uncertainty over admissible predictions at a fixed confidence level. We use this feature in Section \ref{sec:distribution_quality} to explore potential realizations of hurricane trajectories over targeted confidence ranges.

\subsection{Risk Controlling bands}
Figure \ref{fig:reconfomalize} shows the upper band for each method on each setting, showing how re-conformalization adapts to asymmetry, mean biases, and heteroskedasticity across the domain. Thus, even when the underlying model is heavily misspecified, CPD provides tight, exactly controlled, prediction bands. The four panels correspond to the four columns in Table \ref{tab:risk_control}, with the exact data generation defined in Section \ref{apx:sec:data} under 1D Gaussian processes.

\begin{figure}[h]
\centering
\includegraphics[width=\textwidth]{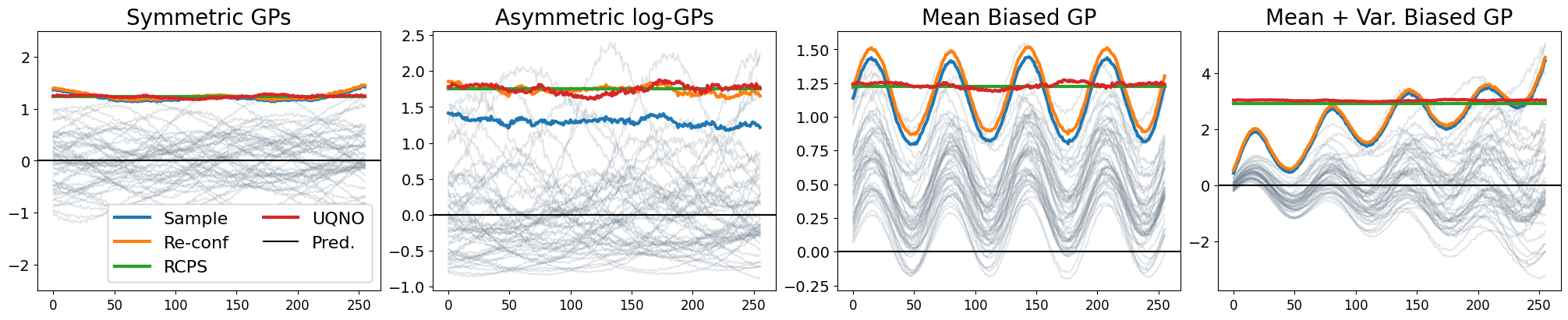}
\caption{Sampled prediction bands (Sample), Reconformalized sampled bands (Re-conf), and RCPS bands from each scenario in Table \ref{tab:risk_control}. Samples provide the shape, reconformalization fixes the width.}
\label{fig:reconfomalize}
\end{figure}

\subsection{Naive v.s. Repulsive sampling}
\begin{wraptable}[6]{r}{0.4\columnwidth}
    \vspace{-1.5\baselineskip}
    \caption{Avg. nearest neighbor distance}
    \centering
    \begin{tabular}{lccccc}
        \toprule
         Dim. $p = $ & $10$ & $25$ & $50$ & $100$ \\
        \midrule
         Naive & 0.624 & 0.964 & 1.124 & 1.225  \\
        	Repulse & 2.150 & 4.604 & 6.296 & 7.591 \\
        \bottomrule
    \end{tabular}
    \label{tab:repulsion}
\end{wraptable}
Table \ref{tab:repulsion} shows the average smallest pairwise Euclidean distance between boundary samples ($n = 1000$) under naive sampling (Alg \ref{alg:sampling}) and repulsive sampling (Alg. \ref{alg:repulsion}) at ($\alpha = 0.1$) on the isotropic gaussian task with and an $\ell_2$ score. Base measure is $N(p^{1/2}, I_p)$ to induce non-uniform sampling. $K = 50$ repulsive steps at $\delta = 0.1$ increases spread by 4-5$\times$ across $p$. Results averaged over 25 simulations.

\subsection{Bootstrap Standard Errors}

Table \ref{tab:baselines} reports Energy Distance, LSD, and MMD scores over a single test set. Here we report the bootstrapped standard errors in Tables \ref{tab:baselines_group1} and \ref{tab:baselines_group2}.

\begin{table}[h]
    \caption{Bootstrap standard errors for Table \ref{tab:baselines} on the comparison of predictive distribution and conformal predictive distribution (CPD) baselines.}
    \centering
    \small
    \setlength{\tabcolsep}{4pt}
    \begin{tabular}{l ccc ccc ccc ccc}
    \toprule
    	 & \multicolumn{3}{c}{\bf GP Regression} &
	     \multicolumn{3}{c}{\bf Elliptic PDE Inv.}  & 
	     \multicolumn{3}{c}{\bf Navier Stokes}  \\
	     \cmidrule(lr){2-4} 
	     \cmidrule(lr){5-7}
	     \cmidrule(lr){8-10}
    Method		& ED  & LSD & MMD  & ED  & LSD & MMD  & ED  & LSD & MMD   \\
    \midrule
    CPD-G 		& 0.0023  & 0.0006 & 0.0011 & 0.0001& 0.0004  & 0.0007 & 0.0047 & 0.0091 & 0.0014 \\
    CPD-L 		& 0.0023  & 0.0006 & 0.0014 & 0.0001 & 0.0004 & 0.0007 & 0.0050 & 0.0015 & 0.0013 \\ 
    Drop.         	& 0.0029  & 0.0010 & 0.0005 & 0.0006 & 0.0015  & 0.0005 & 0.0064 & 0.0035 & 0.0030 \\
    D. Ens.		& 0.0022  & 0.0009 & 0.0010 & 0.0007 & 0.0013  & 0.0010 & 0.0056 & 0.0034 & 0.0015 \\
    IQN             	& 0.0027  & 0.0010 & 0.0010 & 0.0008 & 0.0017  & 0.0011 & 0.0058 & 0.0022 & 0.0030 \\
    Flow              & 0.0031 & 0.0020 & 0.0006 & 0.0006 & 0.0019  & 0.0008 & 0.0051 & 0.0019 & 0.0013 \\
    MDN         	& 0.0031  & 0.0067 & 0.0009 & 0.0010 & 0.0091 & 0.0011 & 0.0051 & 0.0065 & 0.0026 \\ 
    MVE		& 0.0031  & 0.0061 & 0.0008 & 0.0010 & 0.0990  & 0.0003 & 0.0051 & 0.0059 & 0.0026 \\
    \bottomrule
    \end{tabular}
    \label{tab:baselines_group1}
\end{table}

\begin{table}[h]
    \caption{Bootstrap standard errors for Table \ref{tab:baselines} on the comparison of predictive distribution and conformal predictive distribution (CPD) baselines.}
    \centering
    \small
    \setlength{\tabcolsep}{4pt}
    \begin{tabular}{l ccc ccc ccc ccc ccc ccc}
    \toprule
    	 &  \multicolumn{3}{c}{\bf Precip. Downscale} & 
	     \multicolumn{3}{c}{\bf Climate Debias} \\
	     \cmidrule(lr){2-4} 
	     \cmidrule(lr){5-7}
    Method		& ED  & LSD & MMD & ED  & LSD & MMD  \\
    \midrule
    CPD-G 		& 0.0064 & 0.0018 & 0.0009 & 0.0019 & 0.0001 & 0.0014  \\
    CPD-L 		 & 0.0072 & 0.0015 & 0.0008 & 0.0022 & 0.0002 & 0.0013 \\
    Drop.         	 & 0.0070 & 0.0020 & 0.0012 & 0.0010 & 0.0001 & 0.0009 \\
    D. Ens.		 & 0.0070 & 0.0024 & 0.0014 & 0.0014 & 0.0001 & 0.0011 \\
    IQN             	 & 0.0062 & 0.0023 & 0.0013 & 0.0015 & 0.0001 & 0.0014 \\
    Flow              & 0.0057 & 0.0021 & 0.0006 & 0.0014 & 0.0001 & 0.0007 \\
    MDN         	&  0.0052 & 0.0032 & 0.0010 & 0.0011 & 0.0001 & 0.0011 \\
    MVE		&  0.0044 & 0.0027 & 0.0012 & 0.0018 & 0.0004 &  0.0012 \\
    \bottomrule
    \end{tabular}
    \label{tab:baselines_group2}
\end{table}

\section{Simulation Details} \label{apx:sec:details}
\subsection{Datasets and processing} \label{apx:sec:data}

\paragraph{Figure \ref{fig:convergence}: Isotropic Gaussian:} 
We generate linear regression data with Gaussian covariates and noise. For each ambient dimension $p\in\{10,20,\dots,100\}$, we sample a random coefficient matrix $\Theta\in\mathbb{R}^{p\times p}$ with independent standard normal entries. Training, calibration, and test inputs $X\in\mathbb{R}^{n\times p}$ are drawn with i.i.d.\ standard normal rows, and outputs are generated according to
\[
Y = X\Theta + \varepsilon,
\]
where $\varepsilon\in\mathbb{R}^{n\times p}$ has i.i.d.\ standard normal entries. We use $n=1000$ samples per split, with each split generated independently under the same specification.

\paragraph{Figure \ref{fig:convergence}: Anisotropic t-Dist.:} 
For the heavy–tailed vector experiments, we use the same linear regression setup as above but replace the Gaussian noise with a heteroskedastic Student-$t$ distribution. For each dimension $p\in\{10,20,\dots,100\}$, we draw a random coefficient matrix $\Theta\in\mathbb{R}^{p\times p}$ with i.i.d.\ standard normal entries and sample inputs $X\in\mathbb{R}^{n\times p}$ with independent standard normal rows. Outputs are generated as
\[
Y = X\Theta + \varepsilon,
\]
where the noise term has independent components $\varepsilon_{ij}\sim t_{3}\,\sigma_j$ with degrees of freedom $3$ and dimension–dependent scales $\sigma_j=\exp(\alpha_j)$, where $\alpha_j$ increases linearly from $0$ to $1.5$ across coordinates. We use $n=1000$ samples per split, with training, calibration, and test sets generated independently under the same specification.

\paragraph{Figure \ref{fig:convergence}: 2D Gaussian Processes:} 
We generate paired 2D functional data $(X,Y)$ on a spatial grid using a Gaussian process generator with spatial correlation but no temporal dependence. The spatial domain $\mathcal U\subset[0,1]\times[0,1]$ discretized on a regular $p\times p$ lattice with $p=64$. Input fields $X_t$ are drawn as zero–mean Gaussian random fields with separable squared exponential (RBF) covariance and length–scale $\ell_x=0.15$ along both spatial axes, with a small jitter term for numerical stability. Output fields $Y_t$ are generated according to
\[
Y_t(u) = \beta\,X_t(u) + \varepsilon_t(u),
\]
where $\beta=0.6$ and $\varepsilon_t(u)$ is an independent zero–mean Gaussian random field with RBF covariance and length–scale $\ell_y$, which is varied in Figure \ref{fig:convergence} and set to $\ell_y = 0.08$ in Table \ref{tab:baselines}. Samples are i.i.d.\ across $t$, ensuring exchangeability. Training, calibration, and test splits each consist of $n=500$ independently generated realizations under identical settings with different random seeds.

\paragraph{Figure \ref{fig:convergence}: 2D Downscaling:} 
For the 2D downsampling experiment, we use the same Gaussian process generator and data–generation settings as above, including the spatial grid, kernel parameters, and i.i.d.\ sampling across realizations. The difference lies only in the construction of the input–output pairs. Specifically, the high–resolution GP field is treated as the target $Y_t\in\mathbb{R}^{p\times p}$, while the input $X_t$ is obtained by spatially downsampling $Y_t$ to a coarse $q\times q$ grid with $q=8$ using bilinear interpolation. Models are trained to map the low–resolution input field back to the high–resolution target. Training, calibration, and test splits are generated independently under identical settings with  $n=500$ each.

\paragraph{Table \ref{tab:risk_control}: 1D Gaussian processes:}
We generate paired functional data $(X,Y)$ on a one–dimensional spatial grid using the same Gaussian process generator. The domain $\mathcal U\subset[0,1]$ is discretized into $p = 256$ equispaced points. Input fields $X_t$ are sampled as zero–mean Gaussian processes with squared exponential covariance and length–scale $\ell_x=0.2$, while outputs are constructed as
\[
Y_t(u) = \beta\,X_t(u) + \varepsilon_t(u),
\]
with $\beta=0.6$ and $\varepsilon_t$ an independent Gaussian process with RBF covariance and length–scale $\ell_y=0.2$. All samples are independent across $t$. All datasets are generated under identical settings ($n = 500$) with different random seeds for training, calibration, and testing.

\textbf{Asym.}: We center $Y$ across examples at each spatial location, exponentiate pointwise, and then recenter globally,
\[
\tilde Y \;=\; \exp\!\Big(Y - \tfrac{1}{n}\sum_{t=1}^n Y_t\Big)\;-\;\frac{1}{np}\sum_{t=1}^n\sum_{i=1}^p \exp\!\Big(Y_t(u_i)-\tfrac{1}{n}\sum_{s=1}^n Y_s(u_i)\Big),
\]
leaving $X$ unchanged. We use $(X_t,\tilde Y_t)$ in place of $(X_t,Y_t)$ for the asymmetric band experiments.

\textbf{$\Delta \mu$ Bias}: For the training set, the generated outputs are first centered pointwise across samples, exponentiated with reduced amplitude, and globally recentered,
\[
Y_t^{\text{train}} \;=\; \exp\!\big((Y_t-\bar Y)/2\big)\;-\;\frac{1}{n}\sum_{s=1}^n \exp\!\big((Y_s-\bar Y)/2\big),
\]
where $\bar Y(u)=\tfrac{1}{n}\sum_{s=1}^n Y_s(u)$.  For the calibration and test sets, the same transformation is applied but with an additional deterministic mean shift,
\[
Y_t^{\text{cal/test}} \;=\; f_{\text{mean}} \;+\; \exp\!\big((Y_t-\bar Y)/2\big)\;-\;\frac{1}{n}\sum_{s=1}^n \exp\!\big((Y_s-\bar Y)/2\big),
\]
where $f_{\text{mean}}(u) \;=\; 0.5 + 0.25\,\sin(4\pi(2u-1))$ for  $u\in\mathcal U.$
Inputs $X_t$ are left unchanged. This construction preserves exchangeability within each split while inducing a systematic mean mismatch between training and calibration/test, creating a mild but structured form of model bias.

\textbf{$\Delta (\mu, \sigma)$ Bias}: 
As in $\Delta \mu$ Bias, we begin with the same baseline 1D GP data and apply deterministic output transformations to induce model bias, while leaving the inputs unchanged. The training outputs are transformed identically to the $\Delta \mu$ Bias setting,
\[
Y_t^{\text{train}} \;=\; \exp\!\big((Y_t-\bar Y)/2\big)\;-\;\frac{1}{n}\sum_{s=1}^n \exp\!\big((Y_s-\bar Y)/2\big),
\]
with $\bar Y(u)=\tfrac{1}{n}\sum_{s=1}^n Y_s(u)$. For calibration and test, the same nonlinear transformation is applied, followed by both a deterministic mean shift and a spatially varying scale change,
\[
Y_t^{\text{cal/test}} \;=\; f_{\text{mean}} \;+\; \sigma(u)\,Y_t^{\ast},
\]
where $Y_t^{\ast}$ denotes the centered–exponentiated field above, and
\[
f_{\text{mean}}(u) \;=\; 0.5\,\sin(4\pi(2u-1)), 
\qquad
\sigma_j \;=\; 0.5\,\sigma_j^*
\]
where $\sigma_1^*,\ldots,\sigma_p^*$ are $p$ uniformly spaced values from $1$ to $4\pi$.
This construction induces both a structured mean shift and a spatially varying variance shift between training and calibration/test, creating a stronger form of distribution shift than Bias~1 while maintaining exchangeability within each split.

\paragraph{Table \ref{tab:baselines}: 2D Gaussian processes:} We use the same Gaussian process generator and data–generation settings as for Figure \ref{fig:convergence}.

\paragraph{Table \ref{tab:baselines}: Elliptic PDE Inv.:} 
We generate data for a 2D elliptic PDE forward problem on a square interior grid with $p\times p$ points, using $p=32$ and homogeneous Dirichlet (zero) boundary conditions enforced by zero padding. For each realization $i\in\{1,\dots,n\}$ with $n=500$, we sample a forcing field $f_i$ and a spatially varying coefficient field $a_i$, and solve
\[
-\nabla\cdot(a_i\,\nabla u_i) \;=\; f_i
\]
using finite differences and conjugate gradients with tolerance $10^{-6}$ (maximum $500$ iterations). The coefficient field $a_i$ is a smooth log-normal random field obtained by spectral filtering of white noise and exponentiation, while the forcing $f_i$ is a band-limited Gaussian field constructed by masking Fourier modes with $\|k\|^2\le k_{\max}^2$ (with $k_{\max}=8$), inverse transforming, and rescaling to have standard deviation $0.1$. The resulting solution fields $\{u_i\}$ are standardized across the dataset to have unit marginal variance. We treat the pairs $(f_i,u_i)$ as independent samples; training, calibration, and test splits are generated under identical settings with independent random seeds.

\paragraph{Table \ref{tab:baselines}: Navier Stokes Approximation:} 
We generate paired data from the 2D incompressible Navier--Stokes equations in vorticity form on the periodic domain $[0,1]^2$, discretized on a $p\times p$ grid with $p=64$. For each split (training, calibration, and test) we generate $n=500$ independent realizations under identical settings using different random seeds. For each sample $i$, we draw a smooth initial vorticity field $\omega_{0,i}$ by spectrally filtering white noise with a hard cutoff $k_{\max}=20$ and additional exponential smoothing, and we draw a band--limited random forcing field $f_i$ supported on Fourier modes with $4\le \lVert k\rVert\le 8$, which is held constant in time over the trajectory. The viscosity $\nu_i$ is sampled independently from a log--uniform prior on $[10^{-4},\,5\times10^{-3}]$ and is provided to the model as a constant spatial field. The input is thus
\[
X_i(u)=\big(\omega_{0,i}(u),\,\nu_i\big),
\]
and the output is the terminal vorticity field
\[
Y_i(u)=\omega_i(T,u),
\]
where $\omega_i$ evolves according to
\[
\partial_t\omega + u\cdot\nabla\omega = \nu\,\Delta\omega + f,
\qquad
u=(\partial_y\psi,-\partial_x\psi),\ \ \Delta\psi=\omega.
\]
Time integration is performed using a pseudospectral scheme with FFT-based derivatives, $2/3$ dealiasing of the nonlinear term, and fourth-order Runge--Kutta time stepping with $\Delta t=10^{-3}$ for $2000$ steps.

\paragraph{Table \ref{tab:baselines}: Precipitation Downscaling:} 
We use gridded total precipitation fields (ERA5 reanalysis) $y_t$ (\texttt{tp}) on a latitude--longitude grid of size $181\times 301$ over the domain $\text{lat}\in[15,60]$ and $\text{lon}\in[-135,-60]$, with $1032$ monthly snapshots from 1940--01 to 2025--12. We form a single-channel field and convert to millimeters via $y_t \leftarrow 1000\cdot y_t$, then apply a pointwise \texttt{softplus} transform to stabilize near-zero values. To define a downscaling task, we first resample each high-resolution field to an intermediate grid $64\times 128$ using nearest-neighbor interpolation,
\[
y_t \;\leftarrow\; \mathrm{resize}_{\mathrm{nn}}\!\big(y_t,\;64\times128\big),
\]
and then construct the low-resolution input by further nearest-neighbor downsampling to $8\times 16$,
\[
x_t \;=\; \mathrm{resize}_{\mathrm{nn}}\!\big(y_t,\;8\times16\big).
\]
Finally, we normalize each pair using statistics of the corresponding low-resolution field: we subtract the spatial mean of $x_t$ from both $x_t$ and $y_t$, and divide both by the spatial standard deviation of $x_t$ (with broadcasting over the grid). We use $n_{\text{train}}=600$ and $n_{\text{cal}}=300$ time points for training and calibration respectively, with the remaining snapshots held out for testing.

\paragraph{Table \ref{tab:baselines}: Climate Debiasing:} For the climate debiasing task, we use 900 global 2m surface temperature fields (1940-2014) from the CESM climate model \citep{hurrell2013community} on a grid size of $64 \times 128$ as inputs, and target global 2m surface temperature fields from ERA5 reanalysis on the same grid \citep{hersbach2020era5}. We take $n_{\text{train}}=600$ and $n_{\text{cal}}=200$ time points for training and calibration respectively, with the remaining snapshots held out for testing. Each data product (CESM and ERA5) is individually anomalized into its forced response by subtracting off the training monthly averages and then further divided by the training global spatial standard deviation to ensure roughly unit scale. 

\paragraph{Figure \ref{fig:hurr}: Hurricane Trajectory Forecasting:}
We use the HURDAT2 Atlantic basin best-track dataset from the National Hurricane Center (NHC), covering 1851--2025 and recording 6-hourly observations of latitude, longitude, maximum sustained wind speed, minimum sea-level pressure (SLP), storm status, and (for modern storms) 34/50/64-knot wind radii in four quadrants \citep{landsea2013atlantic, harris2021elastic}. Each timestep is encoded as a 24-dimensional feature vector comprising these quantities together with missingness flags for wind, SLP, and each radii threshold.

We retain only storms from 2004 onward to ensure reliable auxiliary observations, requiring at least $95\%$ non-missing SLP entries and at least some 34-knot wind radii data per storm. From each qualifying storm we extract sliding windows with a context length of $L_c=24$ steps (144\,h) and a prediction horizon of $L_p=12$ steps (72\,h), using a stride of $6$ steps and discarding storms with fewer than $48$ total observations. Splitting is performed at the storm level to prevent window-level leakage: storms are randomly assigned (seed~$=0$) to 80/10/10 train/calibration/test splits.

For preprocessing, observed positions are converted to local equirectangular $(x,y)$ coordinates in kilometers, centered at the last context point. The input sequence $x_{\mathrm{ctx}} \in \mathbb{R}^{L_c \times F_{\mathrm{in}}}$ consists of the local $(x,y)$ positions concatenated with auxiliary covariates: wind speed, SLP, and their missingness flags ($F_{\mathrm{in}} = 6$). Both inputs and outputs are standardized channelwise using means and standard deviations computed on the training set.

\subsection{Models} \label{apx:sec:models}

\paragraph{Multilayer Perceptron:}
For vector-valued response experiments, we use a simple multilayer perceptron with two hidden layers of width 64 and ReLU activations, mapping $\mathbb{R}^d \to \mathbb{R}^d$. The model is trained with a mean squared error loss using the Adam optimizer and serves as a low-capacity baseline for non-operator regression settings.

\paragraph{2D Fourier Neural Operator:} 

All operator-learning experiments use a common 2D Fourier Neural Operator (FNO) trunk, adapted from the NeuralOperator architecture. Given an input field $x\in\mathbb{R}^{H\times W\times C_{\mathrm{in}}}$, the trunk first applies a $1\times 1$ lifting convolution to map inputs into a latent width-$64$ channel space. The lifted representation is then processed by a stack of $4$ FNO blocks. Each block combines (i) a spectral convolution that applies learned complex-valued linear transformations to a fixed band of low-frequency Fourier modes (using a real FFT in space), (ii) a $1\times 1$ pointwise skip connection in physical space, and (iii) a channel-wise MLP implemented via $1\times 1$ convolutions. Residual connections and GELU nonlinearities are used throughout, with optional dropout for regularization.

\paragraph{Deterministic FNO head.}
The CPD base model uses the shared 2D FNO trunk followed by a single $1\times1$ convolutional head that linearly projects latent features to the output channels. The model is trained end-to-end with a mean squared error loss with the Adam optimizer. This deterministic predictor serves as the point forecast for all CPD variants.

\paragraph{MC Dropout FNO head.}
The MC Dropout model uses the same 2D FNO trunk and deterministic $1\times1$ convolutional head as the base CPD model, but includes dropout ($p = 0.1$) layers within the trunk. The model is trained identically using a mean squared error loss with the Adam optimizer. At inference time, dropout is kept active and multiple stochastic forward passes are used to generate samples, yielding an approximate predictive distribution via Monte Carlo dropout.

\paragraph{Deep ensemble FNO.}
The deep ensemble model consists of $M = 5$ independently initialized deterministic FNOs, all sharing the same trunk and head architecture. Each ensemble member is trained separately using the same mean squared error objective and optimizer settings. At inference time, the ensemble produces multiple predictions by evaluating all members and treating their outputs as samples from the predictive distribution, capturing uncertainty through model diversity.

\paragraph{IQN (quantile function) FNO head.}
The IQN model replaces the deterministic head with a quantile-function head that represents the conditional quantile $Q_x(\tau)$ for $\tau\in(0,1)$. Given trunk features $h\in\mathbb{R}^{H\times W\times \text{width}}$ and a quantile level $\tau$, a sinusoidal embedding of $\tau$ is mapped through a linear layer and used to modulate the features in a FiLM-style manner. A final $1\times1$ convolution then produces the corresponding quantile field $q_\tau(x)\in\mathbb{R}^{H\times W\times C_{\text{out}}}$.

Training proceeds by sampling $\tau\sim\mathrm{Unif}(0,1)$ independently for each example and minimizing the pinball (quantile) loss. At inference time, samples from the predictive distribution are obtained by drawing $\tau$ values and evaluating the quantile function, yielding an implicit inverse-CDF sampler.

\paragraph{Mean–variance FNO head.}
The mean–variance model augments the shared 2D FNO trunk with a probabilistic head that outputs a per-pixel Gaussian distribution. A single $1\times1$ convolution maps trunk features to a mean field $\mu(x)$ and a log–standard deviation field $\log\sigma(x)$, assuming diagonal covariance across channels and spatial locations. The model is trained by minimizing the negative log-likelihood of a multivariate Gaussian, and predictive samples are obtained by drawing Gaussian noise scaled by the learned $\sigma(x)$ around the mean.

\paragraph{Mixture density network (MDN) FNO head.}
The MDN model equips the shared 2D FNO trunk with a mixture density head that predicts a per-pixel mixture of $K=5$ diagonal Gaussians. A single $1\times1$ convolution outputs mixture weights, component means, and log–standard deviations for each spatial location and channel. The model is trained by minimizing the negative log-likelihood under the resulting mixture distribution. At inference time, predictive samples are obtained by sampling a mixture component and then drawing from the corresponding Gaussian, allowing the model to represent multi-modal predictive uncertainty.

\paragraph{Conditional Flow Matching (CFM) FNO.}
The conditional flow matching model augments the shared 2D FNO trunk with a velocity-field head and learns a continuous-time generative model for $p(y\mid x)$. Training follows a rectified (OT-style) flow matching scheme: a base sample $y_0$ is drawn from a smooth Gaussian random field, a target $y_1=y$ is drawn from data, and intermediate states are formed by the straight-line bridge $y_t=(1-t)y_0+ty_1$ with $t\sim\mathrm{Unif}(0,1)$. The velocity network $v_\theta(x,y_t,t)$ is trained by mean squared error to match the constant teacher velocity $v^\ast=y_1-y_0$. 

At inference time, samples are generated by integrating the learned velocity field from $t=0$ to $t=1$, starting from $y_0$, using forward Euler with 16 steps. This yields conditional samples without specifying an explicit likelihood, serving as a flexible generative baseline for CPD comparisons.

\paragraph{1D CNN (Trajectories):} 
We train a deterministic 1D convolutional point predictor that maps a fixed-length context window to a fixed-length future trajectory segment in a local Cartesian frame. Each context sequence is first recentered at the last observed location and converted from latitude--longitude to local equirectangular $(x,y)$ coordinates in kilometers. The input sequence $x_{\mathrm{ctx}}\in\mathbb{R}^{L_c\times F_{\mathrm{in}}}$ consists of the local $(x,y)$ positions concatenated with auxiliary covariates (wind, sea-level pressure, and missingness flags). Both inputs and outputs are standardized channelwise using means and standard deviations computed on the training set.

The network is a residual 1D CNN: a $1\times 1$ ``lift'' convolution maps $F_{\mathrm{in}}\!\to\!64$ channels, followed by $4$ residual blocks, each containing two same-padded temporal convolutions with kernel size $k=5$ and GELU nonlinearity. The resulting hidden sequence is aggregated by masked average pooling over time to a fixed-width representation in $\mathbb{R}^{64}$, which is mapped by a final linear layer to $2L_p$ outputs and reshaped to $\hat y\in\mathbb{R}^{L_p\times 2}$ (future local $(x,y)$ coordinates).

Training uses AdamW (learning rate $3\times 10^{-4}$, weight decay $10^{-4}$) for 2000 gradient steps with batch size 128. The loss is a masked trajectory loss that combines position error with a derivative penalty (velocity mismatch) to encourage realistic motion; in our runs we set the velocity weight to $w_{\mathrm{vel}}=2.0$ (and include an additional curvature term as implemented in the shared \texttt{traj\_loss}). At inference time, predicted local $(x,y)$ trajectories are converted back to latitude--longitude using the inverse local projection centered at the last context point.

\subsection{Vector Nonconformity Scores}  \label{apx:sec:vec_scores}

We consider a collection of nonconformity scores ranging from simple norm-based losses to distributional and data-adaptive scores. All scores are computed pointwise in the output space and admit gradients with respect to the prediction, enabling their use within the nonconformity flow framework.

\paragraph{$\ell_2$ score.}
The $\ell_2$ score is defined as the root mean squared error between the prediction $\hat y$ and the target $y$,
\[
s_{\ell_2}(y,\hat y) = \Big( \tfrac{1}{D} \|y-\hat y\|_2^2 \Big)^{1/2},
\]
where $D$ denotes the output dimension. This score induces spherical conformal level sets and yields smooth, globally Lipschitz gradients. It serves as a baseline score in all experiments and corresponds to the geometry most commonly used in regression-based conformal prediction.

\paragraph{$\ell_1$ score.}
The $\ell_1$ score is defined as the mean absolute deviation,
\[
s_{\ell_1}(y,\hat y) = \tfrac{1}{D} \|y-\hat y\|_1.
\]
This score is more robust to outliers than $\ell_2$ but is non-differentiable along coordinate hyperplanes. In practice, this lack of smoothness leads to unstable or stalled flow dynamics, which we observe empirically in the convergence experiments.

\paragraph{Huber score.}
The Huber score interpolates between quadratic and linear penalties,
\[
s_{\mathrm{Huber}}(y,\hat y) =
\frac{1}{D}\sum_{j=1}^D
\begin{cases}
\tfrac{1}{2} r_j^2, & |r_j| < \delta, \\
\delta(|r_j| - \tfrac{1}{2}\delta), & |r_j| \ge \delta,
\end{cases}
\quad r = y-\hat y,
\]
with threshold $\delta>0$. This score retains smoothness near the prediction while reducing sensitivity to large deviations, yielding more stable flows than $\ell_1$ while remaining more robust than $\ell_2$.

\paragraph{$k$NN residual score.}
The $k$NN score measures how similar the prediction residual is to residuals observed on the calibration set. Let $\{r_i\}_{i=1}^n$ denote calibration residuals. The score is defined as the average squared distance to the $k$ nearest residuals,
\[
s_{\mathrm{kNN}}(y,\hat y)
= \frac{1}{k} \sum_{i \in \mathcal{N}_k(r)} \| r - r_i \|_2^2,
\quad r = y-\hat y.
\]
This score adapts to the empirical residual geometry and can capture non-elliptical uncertainty. However, it is only piecewise smooth and induces discontinuous gradients, which again leads to poor convergence behavior in the flow.

\paragraph{Gaussian likelihood score.}
We also consider a parametric Gaussian residual score obtained by fitting a diagonal Gaussian model to calibration residuals,
\[
r = y-\hat y \sim \mathcal{N}(\mu, \Sigma),
\]
and defining the score as the negative log-likelihood,
\[
s_{\mathcal{N}}(y,\hat y) = -\log p(r).
\]
This score yields ellipsoidal conformal sets aligned with the empirical covariance structure and induces smooth, globally well-behaved flows. It corresponds to a Mahalanobis-type geometry in the output space.

\paragraph{Student-$t$ likelihood score.}
Finally, we consider a multivariate Student-$t$ residual model with $\nu$ degrees of freedom,
\[
r \sim t_\nu(\mu,\Sigma),
\]
and define the score as the negative log-density. Compared to the Gaussian score, this induces heavier-tailed conformal sets and improves robustness to extreme residuals. The resulting score remains smooth and differentiable while allowing greater flexibility in the induced uncertainty geometry.

\subsection{2D Operator Nonconformity Scores} \label{apx:sec:2d_scores}

For 2D operator problems, outputs are tensor-valued fields $y,\hat y \in \mathbb{R}^{H\times W\times C}$ (typically $C=1$). We evaluate conformity using scores that probe complementary notions of fidelity: pointwise amplitude error, Sobolev-type smoothness error, spectral-shape error, and multiscale wavelet details error. For each score $S(y,\hat y)$, the split-conformal threshold $\tau_\alpha$ is computed as the empirical $(1-\alpha)$ quantile of calibration scores, and gradients are taken with respect to $\hat y$ for use in the nonconformity flow.

\paragraph{$\ell_2$ field score (CPD-G).}
We use the global root-mean-square error
\[
s_{\ell_2}(y,\hat y)
=
\Big( \tfrac{1}{HWC}\sum_{i=1}^H\sum_{j=1}^W\sum_{c=1}^C (y_{ijc}-\hat y_{ijc})^2 \Big)^{1/2}.
\]
This score measures overall amplitude agreement in pixel space and induces spherical level sets in the ambient Euclidean geometry of the discretized field.

\paragraph{Sobolev score.}
To penalize both pointwise error and mismatch in local spatial variation, we use an $H^1$-like discrepancy based on first-order finite differences of the residual field $e = y-\hat y$:
\[
s_{\mathrm{Sob}}(y,\hat y)
=
\Big(\mathbb{E}[e^2]
+
\lambda \big(\mathbb{E}[(\Delta_x e)^2]+\mathbb{E}[(\Delta_y e)^2]\big)
\Big)^{1/2},
\qquad \lambda=1,
\]
where expectations denote spatial-channel averages and $\Delta_x,\Delta_y$ are periodic forward differences implemented by one-step rolls along width and height. This score discourages over-smoothing and small-scale artifacts by explicitly controlling gradient energy.

\paragraph{Spectral density score (PSD distance).}
To compare the distribution of energy across spatial frequencies, we compute a power spectral density (PSD) discrepancy. Let $\mathcal{F}$ denote the 2D real FFT over spatial axes and let
\[
P_y(\omega) = \sum_{c=1}^C \big|\mathcal{F}(y)_c(\omega)\big|^2,
\qquad
P_{\hat y}(\omega) = \sum_{c=1}^C \big|\mathcal{F}(\hat y)_c(\omega)\big|^2,
\]
defined on the discrete frequency grid $\omega=(\omega_x,\omega_y)$. The implemented score is an unnormalized $\ell_1$ distance in PSD space,
\[
s_{\mathrm{PSD}}(y,\hat y) = \sum_{\omega} \big|P_{\hat y}(\omega)-P_y(\omega)\big|.
\]
In our experiments we set $K=12$ but operate in the continuous PSD mode (no band-aggregation) and we compare linear power (no log transform). This score is sensitive to spectral misallocation such as excess high-frequency noise or over-suppression of fine scales.

\paragraph{Wavelet multiscale detail score.}
To measure multiresolution discrepancies in local structure, we compute a wavelet-domain score on the residual $e=\hat y-y$ using a 2D Daubechies wavelet (db2) decomposition to depth $J=3$. Let $\{(cH_j,cV_j,cD_j)\}_{j=1}^J$ denote horizontal/vertical/diagonal detail coefficients at scale $j$ for each channel. The score averages the RMS magnitude of detail coefficients across orientations and sums across scales and channels:
\[
s_{\mathrm{wav}}(y,\hat y)
=
\sum_{c=1}^C \sum_{j=1}^J
\frac{1}{3}\Big(
\|cH_{j,c}\|_{\mathrm{RMS}}+\|cV_{j,c}\|_{\mathrm{RMS}}+\|cD_{j,c}\|_{\mathrm{RMS}}
\Big),
\qquad (J=3,\ \text{db2}).
\]
This score emphasizes mismatches in localized features (edges, filaments, texture) across scales, complementing purely spectral or pixelwise norms.

\paragraph{Composite score (max of normalized Sobolev and PSD).}
Finally, we use a composite score that enforces agreement under both a Sobolev-type notion of smoothness and a spectral-shape notion of frequency allocation. Let $s_{\mathrm{Sob}}$ and $s_{\mathrm{PSD}}$ denote the two base scores above. We normalize each component by a calibration-derived scale $\kappa_i$, taken as the median of that score over the calibration set. The composite score is
\[
s_{\mathrm{combo}}(y,\hat y)
=
\max\!\left\{
\frac{s_{\mathrm{Sob}}(y,\hat y)}{\kappa_{\mathrm{Sob}}+\varepsilon},
\frac{s_{\mathrm{PSD}}(y,\hat y)}{\kappa_{\mathrm{PSD}}+\varepsilon}
\right\},
\]
with $\varepsilon>0$ a small numerical constant. Using the max aggregator makes the induced conformal set the intersection of the two (scaled) sublevel constraints, ensuring that samples are simultaneously consistent in both spatial smoothness and spectral energy distribution. This is the primary multi-criterion score used in our 2D operator experiments. Note that the max operation is non-differentiable on the set where both arguments are equal, but this set has measure zero and does not cause convergence issues in practice.

\paragraph{Localized combined score (local $\ell_2$ + log-spectral)  (CPD-L)}
For 2D field outputs $y,\hat y \in \mathbb{R}^{H\times W\times C}$ we also use a \emph{localized} score that combines (i) pixel-space RMSE and (ii) a log-power-spectral discrepancy. Define two base terms
\[
t_{\ell_2}(y,\hat y)
:= \Big(\tfrac{1}{HWC}\|y-\hat y\|_2^2 + \varepsilon\Big)^{1/2},
\qquad
t_{\mathrm{spec}}(y,\hat y)
:= \Big(d_{\mathrm{spec}}(y,\hat y) + \varepsilon\Big)^{1/2},
\]
where $d_{\mathrm{spec}}$ is the mean (over channels) squared error between \emph{log power spectra}. Concretely, for each channel $c$ we compute
\[
S_y^{(c)} := \log\!\big(1 + |\mathcal{F}(y^{(c)}-\overline{y^{(c)}})|^2 + \epsilon_{\mathrm{spec}}\big),\quad
S_{\hat y}^{(c)} := \log\!\big(1 + |\mathcal{F}(\hat y^{(c)}-\overline{\hat y^{(c)}})|^2 + \epsilon_{\mathrm{spec}}\big),
\]
using a 2D real FFT $\mathcal{F}$ over spatial axes and mean-removal $\overline{(\cdot)}$, and set
\[
d_{\mathrm{spec}}(y,\hat y) := \frac{1}{C}\sum_{c=1}^C \mathbb{E}\big[(S_y^{(c)}-S_{\hat y}^{(c)})^2\big].
\]
To prevent either term from dominating, each term is normalized by a robust calibration scale. Let $\{(y_i,\hat y_i)\}_{i=1}^n$ denote calibration pairs and define
\[
\kappa_{\ell_2} := \mathrm{MAD}\big(\{t_{\ell_2}(y_i,\hat y_i)\}_{i=1}^n\big),\qquad
\kappa_{\mathrm{spec}} := \mathrm{MAD}\big(\{t_{\mathrm{spec}}(y_i,\hat y_i)\}_{i=1}^n\big),
\]
where $\mathrm{MAD}(u)=\mathrm{median}(|u-\mathrm{median}(u)|)$ (plus a small numerical jitter). The combined score is then a weighted RMS aggregation,
\[
s_{\mathrm{loc}}(y,\hat y)
:=
\left(
\frac{
w_{\ell_2}\big(t_{\ell_2}(y,\hat y)/\kappa_{\ell_2}\big)^2
+
w_{\mathrm{spec}}\big(t_{\mathrm{spec}}(y,\hat y)/\kappa_{\mathrm{spec}}\big)^2
}{
w_{\ell_2}+w_{\mathrm{spec}}+\varepsilon
}
\right)^{1/2},
\]
with fixed weights $w_{\ell_2}= 10, w_{\mathrm{spec}}=1$ in our experiments.

To obtain \emph{local} conformal thresholds $\tau_\alpha(x)$, we replace the global quantile of calibration scores by a weighted quantile that depends on the input field $x$. We compute a feature map $\phi(x)\in\mathbb{R}^d$ consisting of per-channel mean and standard deviation of $x$ together with a low-resolution pooled representation obtained by resizing $x$ to an $8\times 8$ grid and vectorizing. Let $\{\phi(x_i)\}_{i=1}^n$ be calibration features and standardize them coordinate-wise. For a test input $x_\star$, we form RBF weights
\[
w_i(x_\star) \propto \exp\!\left(-\frac{\|\tilde\phi(x_i)-\tilde\phi(x_\star)\|_2^2}{2h^2}\right),\qquad \sum_{i=1}^n w_i(x_\star)=1,
\]
with bandwidth $h=1$ (on standardized features). Given calibration scores $s_i := s_{\mathrm{loc}}(y_i,\hat y_i)$, the localized threshold is defined as the weighted $(1-\alpha)$ quantile
\[
\tau_\alpha(x_\star) := Q_{1-\alpha}\big(\{s_i\}_{i=1}^n;\{w_i(x_\star)\}_{i=1}^n\big),
\]
implemented by sorting scores and accumulating the normalized weights until the desired mass $1-\alpha$ is reached. This construction concentrates conformity assessment on calibration examples whose inputs are similar to $x_\star$, yielding locally adaptive conformal level sets.

\subsection{Trajectory nonconformity scores} \label{apx:sec:traj_scores}

For hurricane trajectory forecasting, each output is a sequence of planar coordinates
$y,\hat y \in \mathbb{R}^{T\times 2}$ (latitude--longitude projected to a local $(x,y)$ coordinate system for geometry computations). We consider two nonconformity scores: a baseline $\ell_2$ trajectory error and a richer \emph{conditional geometric trajectory} (CGT) score that penalizes mismatches in trajectory shape and dynamics.

\paragraph{$\ell_2$ trajectory score.}
The baseline trajectory score is the global RMS position error,
\[
s_{\ell_2}(y,\hat y)
=
\Big(\tfrac{1}{2T}\sum_{t=1}^T \|y_t-\hat y_t\|_2^2\Big)^{1/2}.
\]
This score is smooth and measures pointwise agreement in position, but it does not directly constrain physically meaningful attributes such as velocity, turning, or total path length.

\paragraph{Conditional geometric trajectory (CGT) score.}
The CGT score is a weighted combination of multiple geometry terms computed from the trajectory and its discrete derivatives. Let first and second differences be
\[
\Delta y_t := y_{t+1}-y_t \in \mathbb{R}^2,\qquad
\Delta^2 y_t := y_{t+2}-2y_{t+1}+y_t \in \mathbb{R}^2,
\]
and define the path length
\[
L(y) := \sum_{t=1}^{T-1}\|\Delta y_t\|_2.
\]
We also define a turning-angle sequence from consecutive velocity vectors. For $v_t=\Delta y_t$, let
\[
\theta_t(y) := \arccos\!\left(
\frac{\langle v_t,v_{t+1}\rangle}{\|v_t\|_2\|v_{t+1}\|_2+\varepsilon}
\right),
\qquad t=1,\dots,T-2,
\]
with clipping inside $\arccos(\cdot)$ for numerical stability. Using these primitives, CGT computes the following RMS-style discrepancy terms between $y$ and $\hat y$:
\[
\begin{aligned}
t_{\mathrm{pos}}(y,\hat y) &:= \Big(\tfrac{1}{2T}\sum_{t=1}^T \|y_t-\hat y_t\|_2^2 + \varepsilon\Big)^{1/2},\\
t_{\mathrm{vel}}(y,\hat y) &:= \Big(\tfrac{1}{2(T-1)}\sum_{t=1}^{T-1}\|\Delta y_t-\Delta \hat y_t\|_2^2 + \varepsilon\Big)^{1/2},\\
t_{\mathrm{curv}}(y,\hat y) &:= \Big(\tfrac{1}{2(T-2)}\sum_{t=1}^{T-2}\|\Delta^2 y_t-\Delta^2 \hat y_t\|_2^2 + \varepsilon\Big)^{1/2},\\
t_{\mathrm{speed}}(y,\hat y) &:= \Big(\tfrac{1}{T-1}\sum_{t=1}^{T-1}(\|\Delta y_t\|_2-\|\Delta \hat y_t\|_2)^2 + \varepsilon\Big)^{1/2},\\
t_{\mathrm{turn}}(y,\hat y) &:= \Big(\tfrac{1}{T-2}\sum_{t=1}^{T-2}(\theta_t(y)-\theta_t(\hat y))^2 + \varepsilon\Big)^{1/2},\\
t_{\mathrm{len}}(y,\hat y) &:= |L(y)-L(\hat y)|.
\end{aligned}
\]
To ensure that no single component dominates due to scale, each term is normalized by a robust calibration scale, computed as a median absolute deviation (MAD) over calibration pairs $\{(y_i,\hat y_i)\}_{i=1}^n$:
\[
\kappa_k := \mathrm{median}\big(|t_k(y_i,\hat y_i)-\mathrm{median}(t_k)|\big)+\varepsilon,\qquad
k\in\{\mathrm{pos},\mathrm{vel},\mathrm{curv},\mathrm{speed},\mathrm{turn},\mathrm{len}\}.
\]
The final CGT score is a weighted RMS aggregation of the normalized terms,
\[
s_{\mathrm{CGT}}(y,\hat y)
=
\left(
\frac{\sum_k w_k\,(t_k(y,\hat y)/\kappa_k)^2}{\sum_k w_k + \varepsilon}
\right)^{1/2},
\]
with equal weights $w_k\equiv 1$ unless otherwise stated. This score remains ``distance-to-$\hat y$'' in spirit but enforces substantially stronger geometric realism by penalizing mismatches in velocity, curvature, turning behavior, and total traveled distance.

\paragraph{Optional geometry-only localization (weighted quantile thresholds).}
In addition to global thresholds, we optionally form input-dependent thresholds by weighting calibration examples according to similarity of the \emph{predicted} future trajectory geometry. Specifically, we extract a low-dimensional feature vector $\phi(\hat y)\in\mathbb{R}^4$ from the predicted future segment,
\[
\phi(\hat y)=\big(\mathrm{mean\_speed},\,\mathrm{std\_speed},\,\mathrm{mean\_curvature},\,\mathrm{mean\_turn}\big),
\]
standardize features across calibration predictions, and define RBF weights
\[
w_i(\hat y_\star)\propto \exp\!\left(-\frac{\|\tilde\phi(\hat y_i)-\tilde\phi(\hat y_\star)\|_2^2}{2h^2}\right),\qquad \sum_i w_i(\hat y_\star)=1.
\]
Given calibration scores $s_i=S(y_i,\hat y_i)$, we set the localized threshold as the weighted $(1-\alpha)$ quantile
\[
\tau_\alpha(\hat y_\star) := Q_{1-\alpha}\big(\{s_i\}_{i=1}^n;\{w_i(\hat y_\star)\}_{i=1}^n\big),
\]
implemented by sorting scores and accumulating weights. This localization uses only trajectory geometry (no wind or pressure covariates) and is intended to mitigate regime heterogeneity by calibrating within geometrically similar forecast patterns.

\subsection{Evaluation Metrics} \label{apx:sec:metrics}

We evaluate predictive distributions using three complementary metrics: Energy Distance (ED), Log Spectral Distance (LSD), and a local patch-based Maximum Mean Discrepancy (MMD). Together, these metrics assess global distributional accuracy, spectral fidelity, and fine-scale spatial structure.

\paragraph{Energy Distance (ED).}
The Energy Distance is a strictly proper scoring rule that measures the discrepancy between two distributions in Euclidean space. Given samples $\{y^{(k)}\}_{k=1}^K \subset \mathbb{R}^D$ from a predictive distribution and a realized target $y \in \mathbb{R}^D$, we compute
\[
\mathrm{ED} = \frac{1}{K} \sum_{k=1}^K \| y^{(k)} - y \|_2
\;-\;
\frac{1}{2K^2} \sum_{k,k'} \| y^{(k)} - y^{(k')} \|_2,
\]
where distances are implemented as root-mean-square (RMS) norms to avoid trivial scaling with dimension. ED captures overall distributional alignment and penalizes both bias and lack of diversity. Lower values indicate better agreement between the predictive distribution and the data-generating process.

\paragraph{Log Spectral Distance (LSD).}
To assess whether predictive samples reproduce the correct allocation of energy across spatial frequency scales, we compute a Log Spectral Distance based on Fourier power spectra. For each field $u$, we compute the centered 2D discrete Fourier transform and its power spectrum $S(u) = |\mathcal{F}(u)|^2$. Let $\beta$ denote the mean spectral power of the target field. We define the log-amplitude spectrum
\[
\tilde{S}(u) = \log\!\left(1 + \frac{S(u)}{\beta}\right),
\]
and compute the LSD as the mean squared error between the target spectrum and the spectra of the predictive samples. This metric is sensitive to over-smoothing, spurious high-frequency noise, and spectral misallocation, while remaining invariant to global amplitude shifts.

\paragraph{Patch-wise Maximum Mean Discrepancy (MMD).}
To measure local texture and fine-scale structural fidelity, we compute a patch-based Maximum Mean Discrepancy between the target field and the predictive samples. We extract overlapping $p \times p$ spatial patches on a regular grid and treat the resulting patch vectors as samples from local feature distributions. Patches are optionally standardized to remove absolute intensity information, focusing the comparison on relative structure and texture.

We compute the squared MMD using a Gaussian RBF kernel,
\[
\mathrm{MMD}^2(A,B) =
\mathbb{E}[k(a,a')] + \mathbb{E}[k(b,b')] - 2\mathbb{E}[k(a,b)],
\]
where $A$ denotes patches from the target field and $B$ denotes pooled patches from all predictive samples. An unbiased estimator is used by removing diagonal kernel terms. To improve robustness, we average the MMD across multiple kernel bandwidths. This metric penalizes mismatches in local spatial organization that are not detectable through global norms or spectral summaries.

\section{Qualitative Samples (Table \ref{tab:baselines})}\label{apx:sec:qualitative}
To visually demonstrate the sample quality of each method, beyond what metrics alone can show, we provide a few samples from each method under each data setting. Each column in each plot represents the output of a UQ method for the given data, except the first column which shows actual realizations from the test set. Across all data modalities, CPDs tend to produce samples that visually represent the target process, as indicated by the relatively lower LSD and MMD scores in Table \ref{tab:baselines}.

On the 2D GP prediction (Figure \ref{fig:gp_samples}) and the temperature debiasing tasks (Figure \ref{fig:temp_samples}) all models are able to reproduce the target process fairly well given their simplistic structure, with the biggest differentiations being whether the model oversmooths (D. Ens., IQN) or injects excessive high frequency noise (Dropout, MVE, MDN). On more complex tasks, such as the Navier Stokes approximation (Figure \ref{fig:ns_samples}) or precipitation downscaling (Figure \ref{fig:precip_samples}), we can see that not only do baselines grossly underestimate variability, they also do not produce accurate fine-scale structures. These examples are dominated by their structural content, rather than overall magnitude, rendering many baselines largely ineffective. The elliptic PDE inverse modeling problem (Figure \ref{fig:elliptic_samples}) presents a unique challenge in that the boundary conditions we are trying to predict are Gaussian processes, but the baselines cannot seem to capture the scale correctly.

\begin{figure}[h!]
\centering
\includegraphics[width=0.9\textwidth]{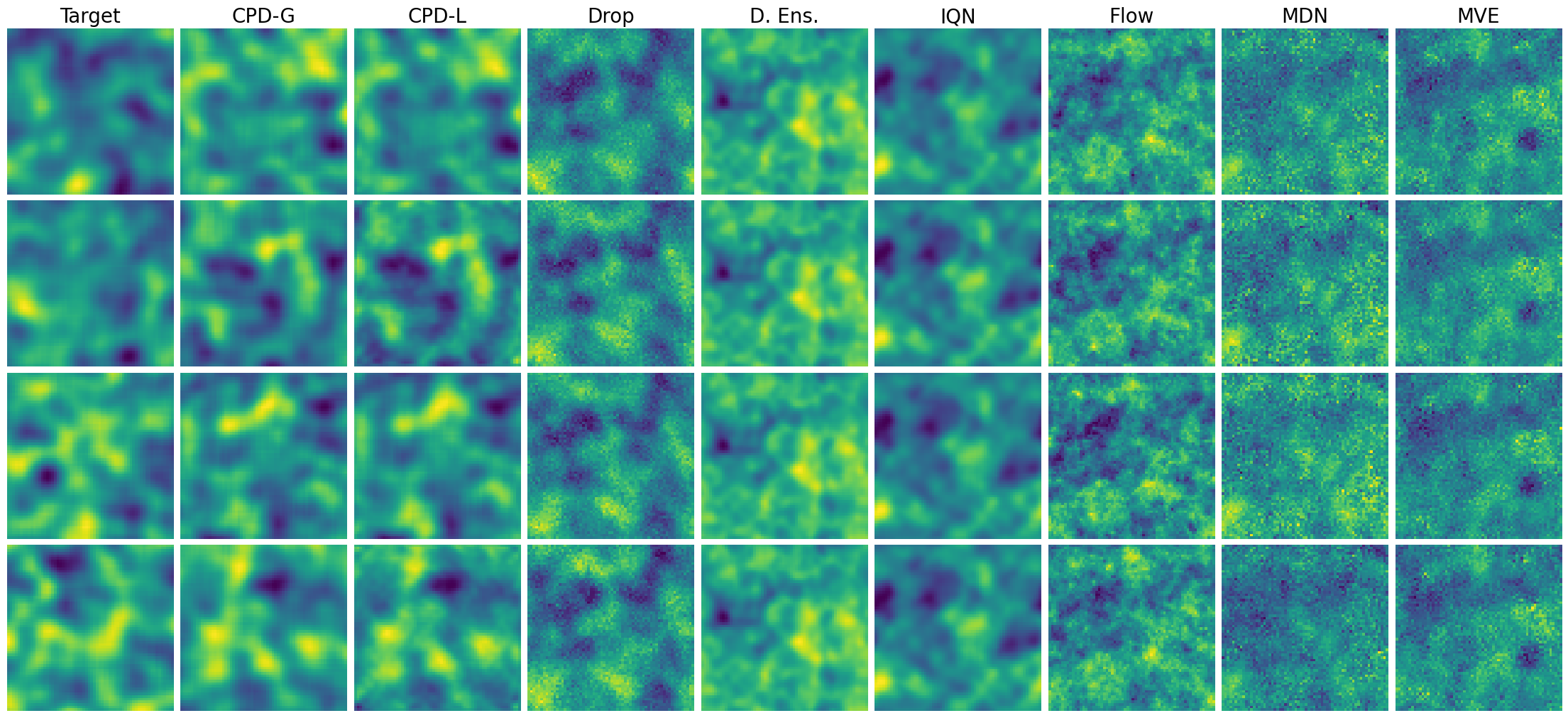}
\caption{Samples from the 2D GP regression experiment. Column one shows realizations from the response $Y$, while the remaining columns show samples from each UQ method.}
\label{fig:gp_samples}
\end{figure}

\begin{figure}[h!]
\centering
\includegraphics[width=0.9\textwidth]{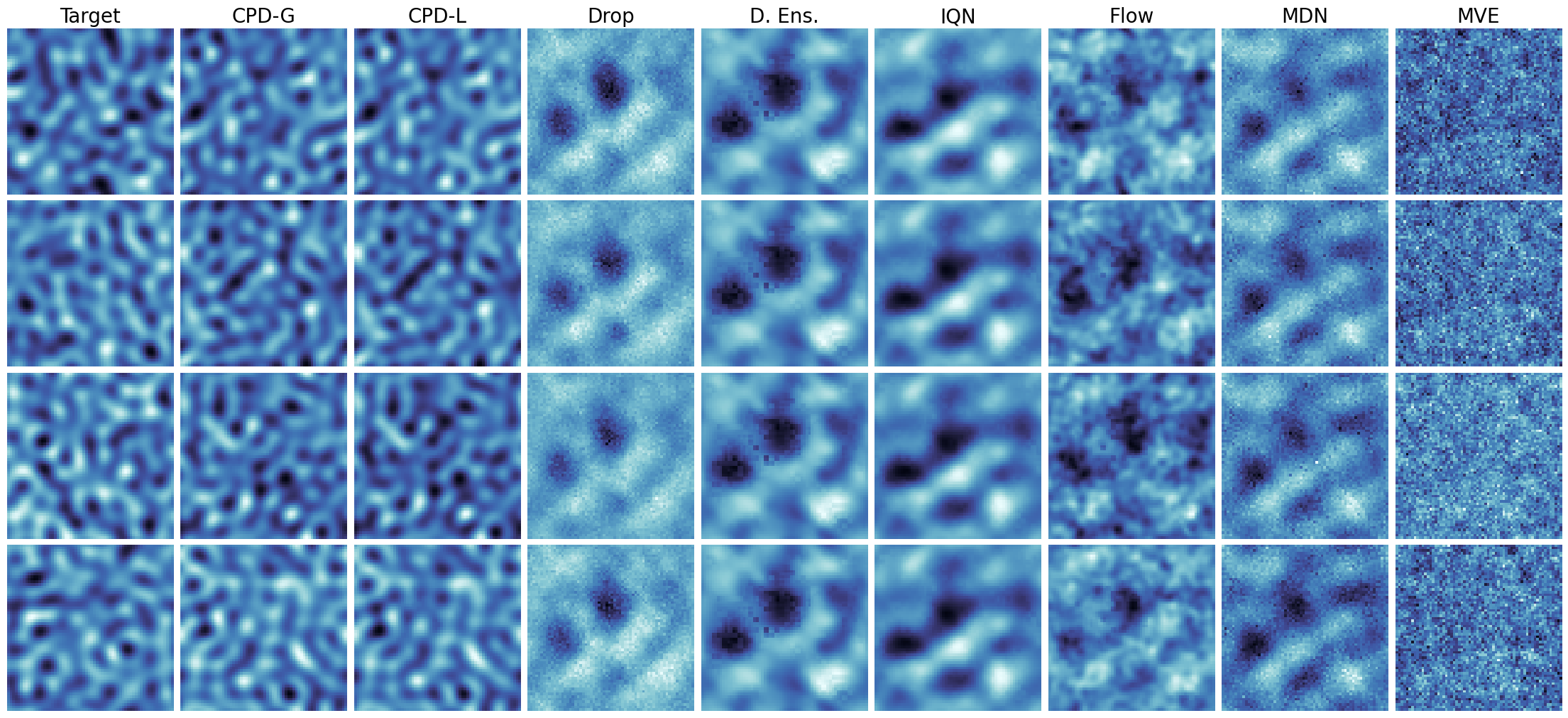}
\caption{Samples from the Elliptic PDE inversion experiment. Column one shows realizations from the response $Y$, while the remaining columns show samples from each UQ method.}
\label{fig:elliptic_samples}
\end{figure}

\begin{figure}[h!]
\centering
\includegraphics[width=0.9\textwidth]{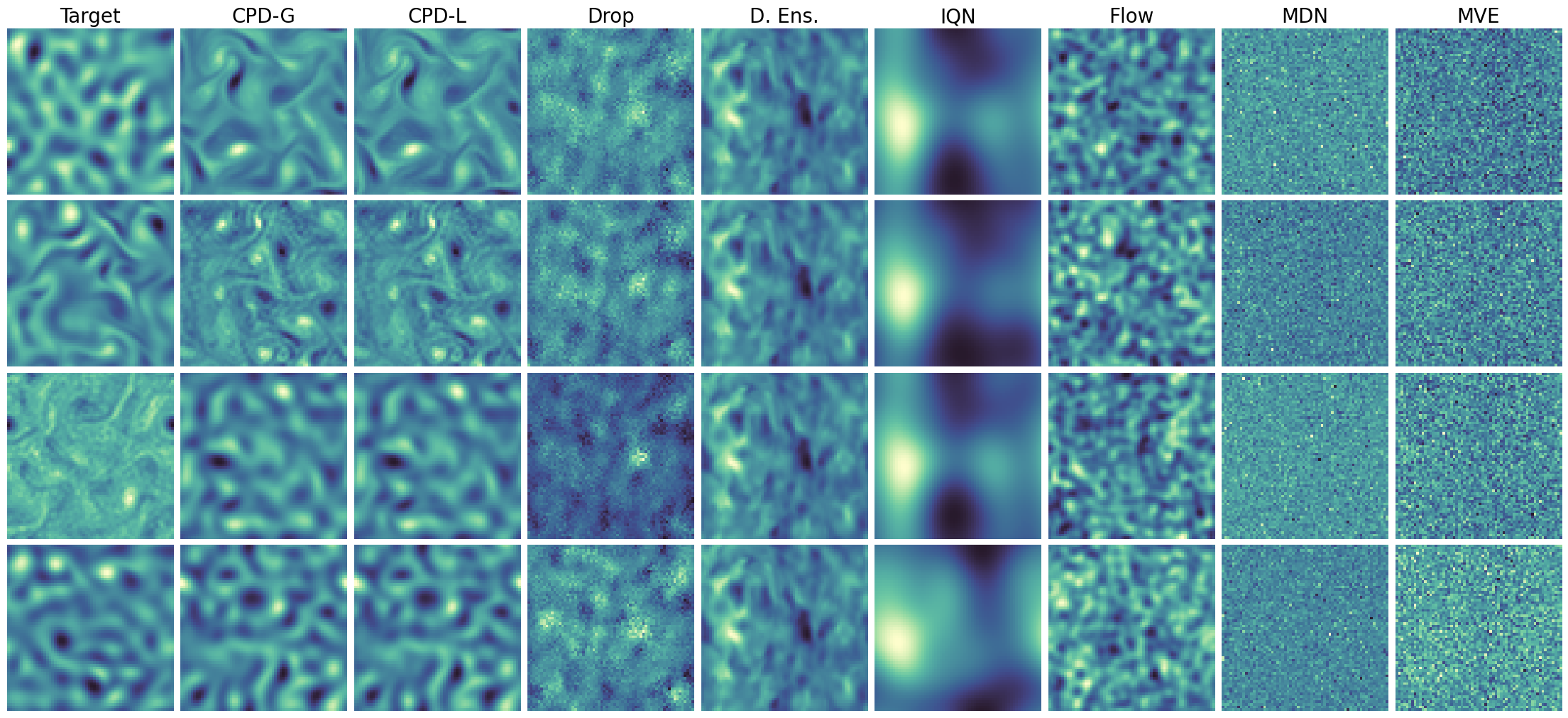}
\caption{Samples from the Navier-Stokes approximation experiment. Column one shows realizations from the response $Y$, while the remaining columns show samples from each UQ method.}
\label{fig:ns_samples}
\end{figure}

\begin{figure}[h!]
\centering
\includegraphics[width=0.9\textwidth]{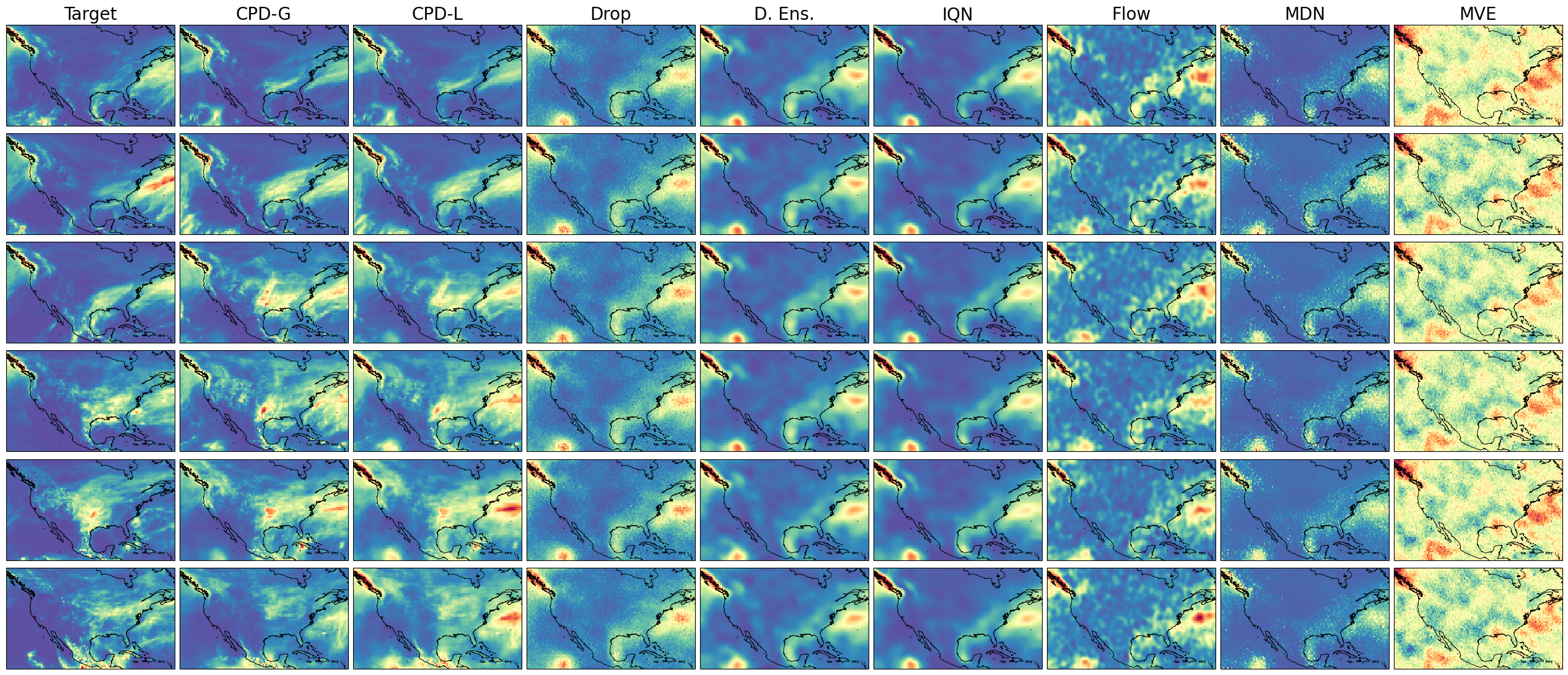}
\caption{Samples from the precipitation downscaling experiment. Column one shows realizations from the response $Y$, while the remaining columns show samples from each UQ method.}
\label{fig:precip_samples}
\end{figure}

\begin{figure}[h!]
\centering
\includegraphics[width=0.9\textwidth]{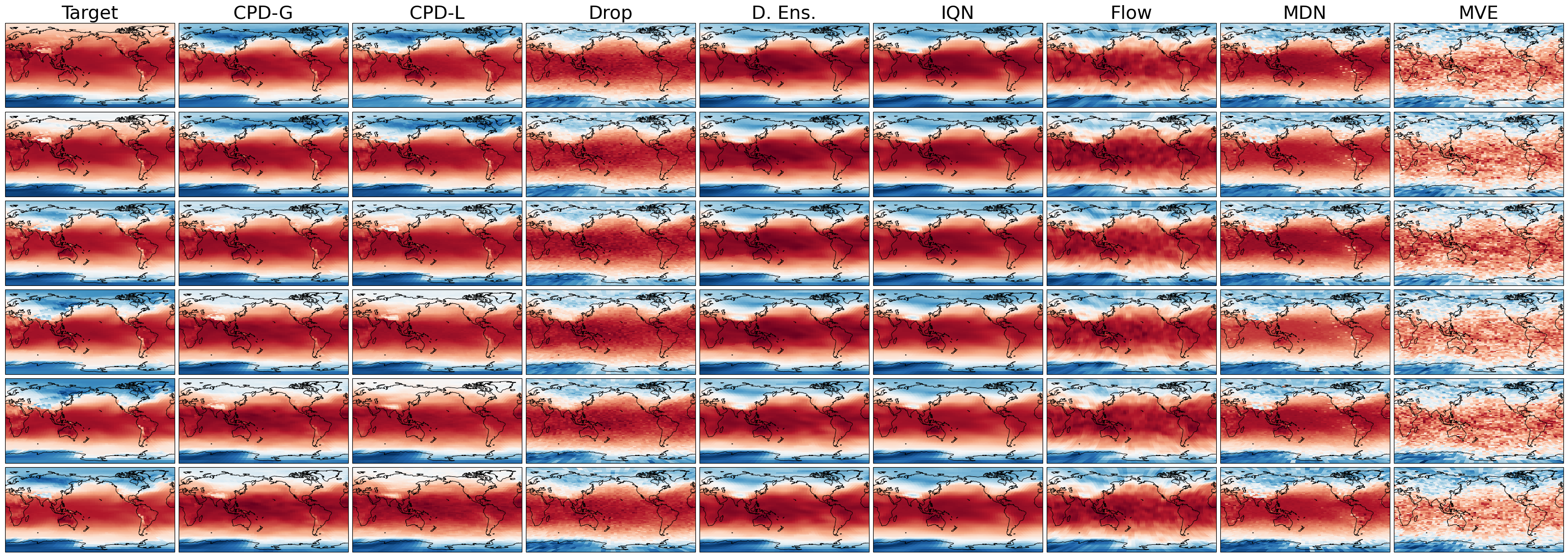}
\caption{Samples from the climate model debiasing experiment. Column one shows realizations from the response $Y$, while the remaining columns show samples from each UQ method.}
\label{fig:temp_samples}
\end{figure}


\end{document}